\documentclass[10pt]{article}
\usepackage[preprint]{tmlr}
\usepackage{amsmath,amssymb,amsthm,bm,mathtools}
\usepackage{booktabs}
\usepackage{array,tabularx}
\usepackage{graphicx}
\usepackage{xcolor}
\usepackage[hypertexnames=false]{hyperref}
\hypersetup{hidelinks}
\usepackage{algorithm}
\usepackage{algpseudocode}
\usepackage{float}
\usepackage{enumitem}
\usepackage{microtype}

\newcommand{\inner}[2]{\left\langle #1,#2\right\rangle}
\newcommand{\norm}[1]{\left\lVert #1\right\rVert}

\title{Optimizer Memory Makes Shuffle Order\\ a First-Order Source of Fine-Tuning Noise}
\author{\name John Sweeney \email john.sweeney@sideplane.ai \\ \addr Sideplane.ai}
\newtheorem{theorem}{Theorem}
\newtheorem{proposition}{Proposition}

\theoremstyle{definition}
\newtheorem{definition}{Definition}
\newtheorem{assumption}{Assumption}

\begin{document}
\maketitle
\newcommand\blfootnote[1]{\begingroup\renewcommand\thefootnote{}\footnote{#1}\addtocounter{footnote}{-1}\endgroup}
\blfootnote{An LLM-based tool was used for copyediting, wording, and writing experiment and analysis code. The study design, claims, and results are the authors' own, verified by the authors.}

\begin{abstract}
Shuffle order can be a larger source of fine-tuning noise than a memoryless analysis predicts: fixed-clock optimizer memory makes local equal-multiset contrasts first order in the learning rate rather than second order, and the resulting order channel can be large enough for a single seed to flip a close A/B comparison. We isolate this mechanism and derive a fit-free way to size the noise it produces. For a memoryless optimizer, reordering an equal multiset has no first-order endpoint term; the leading local contrast is the $O(\eta^2)$ gradient bracket. Fixed-clock optimizers such as AdamW are different. Their moment buffers, preconditioner state, and de-biasing counters advance with the step index rather than with the learning-rate-scaled time $\tau=\eta k$, so the same gradient can receive a position-dependent endpoint weight. For any fixed finite measurement window, a lifted-state expansion gives an $O(\eta)$ equal-multiset contrast whenever the first-order replay coefficient is nonzero, while regular and clock-matched controls remain $O(\eta^2)$; a bare fixed-$\beta$ momentum buffer is already enough. A bitwise-deterministic replay from one warmed optimizer state isolates the mechanism, giving order-variance slopes $1.83$ for AdamW, $2.00$ for fixed-$\beta$ momentum, and $4.00$ for SGD; matching the memory clock to $\tau$ restores the regular exponent. For AdamW with a frozen preconditioner, the same impulse-weight kernel gives a closed-form asymptotic order-variance floor after the local potentials are measured, with no fitted coefficients. The result is local to the measurement window---independent reshuffling can average the channel across windows---but it yields order-noise error bars, positional attribution weights, and a seed-budget criterion for fine-tuning comparisons.
\end{abstract}

\section{Introduction}
\label{sec:intro}

How much does the order of fine-tuning data change where training ends up?  For a memoryless gradient step, very little: reordering an equal multiset leaves the first-order term unchanged, so the contrast is second order in the learning rate and the gradient-bracket residual \citep{dherin2023implicit,rukhovich2025commute,sweeney2026liebracket} cancels under reshuffling.  The optimizers used in practice are not memoryless.  Their moment buffers and de-biasing counters advance with position, so the same gradient gets a position-dependent weight and the contrast no longer cancels at first order.  Within this local window, fixed-clock state creates an $O(\eta)$ order channel rather than only an $O(\eta^2)$ one; when that channel is large relative to the gap between two configurations, it can flip a close comparison, and the same expansion that predicts it also tells you how to size it.

AdamW-style training changes the state being expanded.  In the small-$\eta$ limit over a fixed window, the parameters remain close to the measurement state, so the gradients are effectively frozen, but the moment buffers, preconditioner state, and de-biasing counters still advance with position.  The same gradient can therefore receive different first-order endpoint weights depending on where it appears.  For plain momentum the picture is concrete: an early gradient persists in the buffer for more steps than a late one, so position alone changes its contribution to the endpoint.

We study this local finite-window limit at a fixed measurement state.  Regular optimizers have flat first-order impulse weights and $O(\eta^2)$ equal-multiset contrasts.  Fixed-clock optimizer state can have non-flat impulse weights and $O(\eta)$ contrasts whenever the first-order replay coefficient is nonzero.  Here ``first-order'' means order $\eta$ in the step-size expansion.  Optimizer state therefore moves equal-multiset order from the memoryless $O(\eta^2)$ bracket regime to an $O(\eta)$ channel, even when the finite-window contrast is small in absolute magnitude.  A fixed-clock buffer is sufficient for this shift; AdamW is the main adaptive optimizer we analyze.

The same impulse-weight kernel has two practical uses.  It sets a local order-noise scale for sizing shuffle-seed comparisons, and it gives a positional attribution factor: a block's influence is its gradient alignment times the finite-clock weight of the position where it appears.  Clock matching is the corresponding control: if the memory and normalization clocks are scaled with the continuous-time training variable $\tau=\eta k$, the first-order impulse profile flattens and the regular $O(\eta^2)$ exponent is restored.

\begin{figure}[t]
\centering
\includegraphics[width=0.82\linewidth]{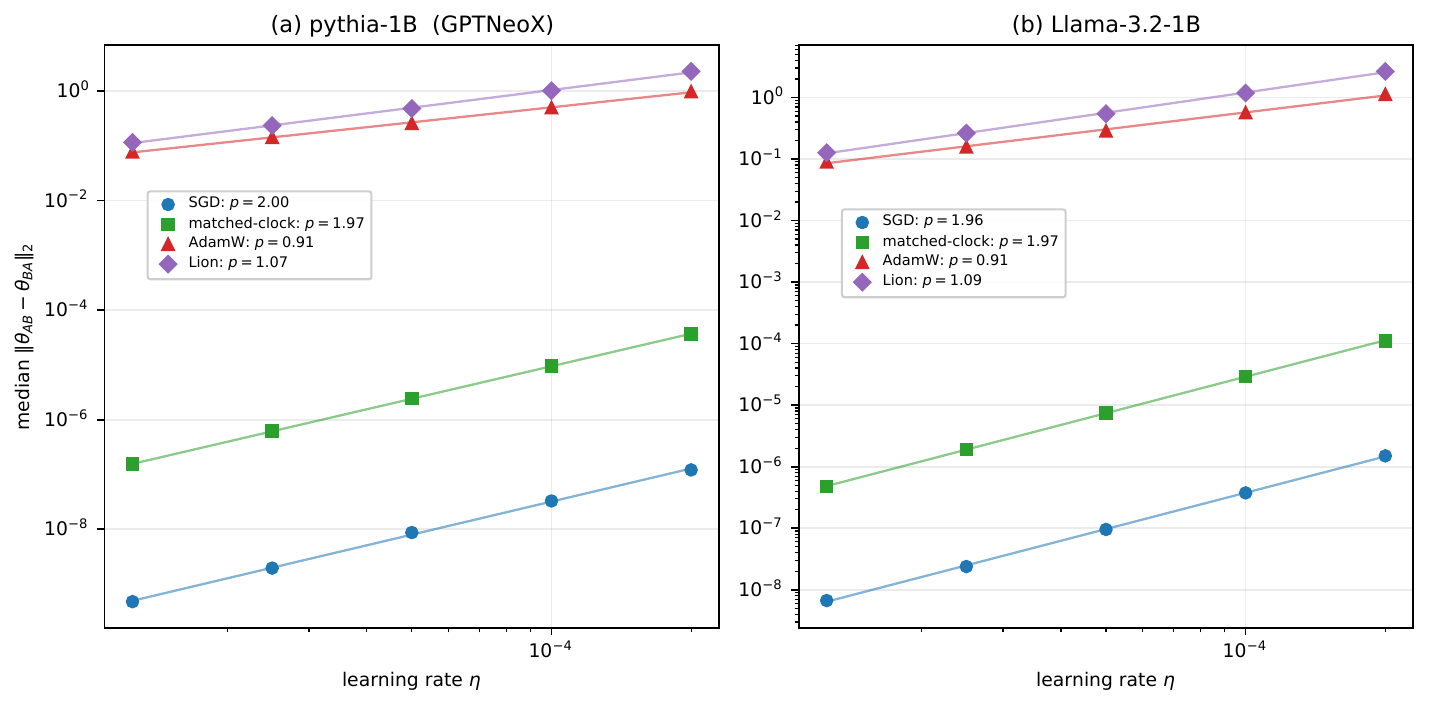}
\caption{Mean-level exponent split.  Equal-multiset order contrasts scale near $\eta^2$ for SGD and matched-clock memory, and near $\eta$ for fixed-clock AdamW/Lion.  The fixed-$\beta$ momentum control in Section~\ref{sec:experiments} isolates fixed-clock buffering from coordinate adaptivity.}
\label{fig:eta-slope}
\end{figure}

For domains $A,B$, window length $K$, and evaluation covector $g_E$, our basic readout is
\begin{equation}
\label{eq:balanced-observable}
    \Delta_E(A,B;K)=\inner{g_E}{\theta_{A^{K/2}B^{K/2}}-\theta_{B^{K/2}A^{K/2}}}.
\end{equation}
The theory gives the first coefficient of this contrast, the leading-order variance over uncontrolled permutations, and local validity diagnostics for the expansion.  Fine-tuning seed sensitivity motivates the measurement \citep{dodge2020finetuning,chen2024order}; we isolate the optimizer-controlled order channel inside that phenomenon.

Table~\ref{tab:scope} summarizes the main results, their mathematical objects, and the scope of evidence.

\begin{table}[t]
\centering
\small
\caption{Main results and scope.  The theorem-level result concerns fixed-clock optimizer state.  The closed-form floor specializes to frozen-preconditioner AdamW and is tested against full frozen-gradient replay, with scalar-sortability ($\rho_{\mathrm{curl}}$) reported as a coefficient-level diagnostic.}
\label{tab:scope}
\begin{tabularx}{\linewidth}{>{\raggedright\arraybackslash}p{0.25\linewidth}>{\raggedright\arraybackslash}p{0.26\linewidth}>{\raggedright\arraybackslash}X}
\toprule
Result & Main object & Scope and evidence \\
\midrule
Fixed-clock state changes the order exponent & Lifted-state replay & Finite window at a fixed measurement state; fixed-$\beta$ momentum and AdamW replay have non-flat impulse profiles. \\
AdamW variance floor & Frozen-$P$ kernel & Closed form for the leading variance term; ordering-orbit tests and $\chi$ compare it with full frozen-gradient replay. \\
Buffering is sufficient & Fixed-$\beta$ momentum & Non-adaptive control separates fixed-clock memory from coordinate adaptivity. \\
Scope diagnostics & $\rho,\chi,\rho_{\mathrm{curl}}$ & Local validity, frozen-$P$ scope, and scalar-sortability tests; large ratios mark cells outside the local regime. \\
Shuffle-seed budget & $\sigma_{\mathrm{ord}}$ & Power calculation for held-out comparisons with order noise. \\
\bottomrule
\end{tabularx}
\end{table}

\section{The lifted-state clock theorem}
\label{sec:clock-theorem}

Let a word $w=(w_1,\ldots,w_K)$ denote the sequence of domains in a finite local window.  The \emph{lifted state} is the pair $z=(\theta,\xi)$, where $\xi$ collects the optimizer variables in addition to the parameters.  We work at a measurement state and take $K$ fixed while $\eta$ varies.  The local question is: when the same domains are seen in a different order, what is the first nonzero coefficient of the resulting parameter difference?

\begin{definition}[Regular optimizer family]
\label{def:regular}
An optimizer family is regular if, for every domain $D$, its local state map satisfies
\begin{equation}
\label{eq:regular-map}
    F_D^\eta(z)=z+\eta X_D(z)+O(\eta^2)
\end{equation}
uniformly on a neighborhood of the measurement state, with $X_D$ independent of $\eta$.
\end{definition}

\begin{proposition}[Regular equal-multiset cancellation]
\label{prop:regular-commutator}
If $F_D^\eta$ is regular and $w,w'$ contain the same multiset of domains, then
\begin{equation}
    F_{w_K}^\eta\cdots F_{w_1}^\eta(z)-F_{w'_K}^\eta\cdots F_{w'_1}^\eta(z)=O(\eta^2).
\end{equation}
For a two-step swap, under the stronger second-order expansion $F_D^\eta=I+\eta X_D+\eta^2Y_D+O(\eta^3)$ of Appendix~\ref{app:proofs}, this leading term is the usual commutator $\eta^2[X_B,X_A]+O(\eta^3)$.
\end{proposition}

\begin{proof}[Proof sketch]
Composing along $w$ gives $z_K(w)=z+\eta\sum_{t=1}^K X_{w_t}(z)+O(\eta^2)$.  The first-order term is a flat sum over the multiset: every generator has coefficient one, independent of position.  Equal multisets therefore cancel at first order.  Expanding two steps gives the displayed Lie bracket.
\end{proof}

An optimizer whose state advances on the step clock need not be regular in the sense of Definition~\ref{def:regular}.  For such fixed-clock state, the local maps are
\begin{equation}
\label{eq:lifted-map}
    \xi_t=G_{D_t}(\xi_{t-1},\theta_{t-1}),\qquad
    \theta_t=\theta_{t-1}+\eta\,U_{D_t}(\xi_t,\theta_{t-1})+O(\eta^2),
    \qquad t=1,\ldots,K.
\end{equation}
The defining feature is that $G_D$ does not converge to the identity with $\eta$: buffers and counters keep changing even when the parameter displacement is scaled down.  Counters are discrete coordinates; the smoothness assumptions below apply only to the continuous arguments of $G_D$ and $U_D$, while the counter advances on the step index rather than on rescaled time $\eta s$.  For a word $w$, define its frozen-parameter optimizer-state replay by
\begin{equation}
\label{eq:frozen-replay-xi}
    \xi^0_t(w)=G_{w_t}(\xi^0_{t-1}(w),\theta_0),\qquad \xi^0_0=\xi_0,
    \qquad t=1,\ldots,K.
\end{equation}

\begin{theorem}[Lifted-state clock theorem]
\label{thm:fixed-clock-singularity}
Assume the leading maps $G_D$ and $U_D$ are $\eta$-independent, $C^1$ in their continuous arguments, and uniformly bounded on a neighborhood of the measurement state for the finite window $K$.  Then, uniformly over words of length $K$,
\begin{equation}
\label{eq:lifted-expansion}
    \theta_K(w)=\theta_0+\eta\sum_{t=1}^{K}U_{w_t}(\xi^0_t(w),\theta_0)+O(\eta^2).
\end{equation}
Consequently, for two equal-multiset words $w,w'$ and readout $u$,
\begin{equation}
\label{eq:lifted-contrast}
    \inner{u}{\theta_K(w)-\theta_K(w')}=\eta\,\Lambda_{w,w'}(u)+O(\eta^2),
\end{equation}
where $\Lambda_{w,w'}$ is the readout difference of the two frozen-parameter optimizer-state replays.  If $\Lambda_{w,w'}(u)\ne0$, the order contrast is $\Theta(\eta)$.  For regular optimizer families, the corresponding first-order replay is multiset-only and vanishes for all equal-multiset contrasts.
\end{theorem}

\begin{proof}[Proof sketch]
Because each parameter step is $O(\eta)$, $\theta_t=\theta_0+O(\eta)$ for every $t\le K$.  Lipschitz continuity of $G_D$ then implies $\xi_t(w)=\xi_t^0(w)+O(\eta)$ by induction: each step sees only an $O(\eta)$ perturbation in the parameter argument and propagates a finite number of such errors.  Substituting $\xi_t^0$ and $\theta_0$ into the parameter update changes each summand by $O(\eta)$, hence the endpoint by $O(\eta^2)$ after multiplication by the step scale.  The substitution yields Eq.~\eqref{eq:lifted-expansion}.  Proposition~\ref{prop:regular-commutator} proves the contrast with regular optimizers.  The full proof is in Appendix~\ref{app:fixed-clock-nondegen}.
\end{proof}

\paragraph{AdamW replay.}
For AdamW, the optimizer state is $\xi=(m,v,s)$ and the fixed-clock replay updates
\begin{align}
    m^+ &= \beta_1m+(1-\beta_1)g_D(\theta), &
    v^+ &= \beta_2v+(1-\beta_2)g_D(\theta)^{\odot2}, &
    s^+ &= s+1,
\end{align}
followed by
\begin{equation}
    U_D(\xi^+,\theta)=
    -\frac{m^+/(1-\beta_1^{s^+})}{\sqrt{v^+/(1-\beta_2^{s^+})}+\epsilon}-\lambda\theta.
\end{equation}
The theorem therefore covers full frozen-gradient AdamW replay, including finite-step $m$ and $v$, both de-biasing clocks, $\epsilon$, and decoupled weight decay, provided the replay preconditioner remains in a smooth region along the finite path.  Concretely, this region is where the de-biased second moment satisfies $v^+/(1-\beta_2^{s^+})\ge\delta>0$; the $\epsilon$ in $\sqrt{v^+/(1-\beta_2^{s^+})}+\epsilon$ bounds the value of the map, not its derivative, so where $\delta$ fails the $C^1$ argument does not apply (see Appendix~\ref{app:fixed-clock-nondegen}).  Sign-based Lion and orthogonalized Muon are nonsmooth, so we treat them as empirical fixed-clock comparisons outside the closed-form AdamW coefficient.

\paragraph{Buffering versus adaptivity.}
Coordinate adaptivity is not required for the shift: fixed-$\beta$ first-moment memory already gives non-flat first-order impulse weights, so nondegenerate equal-multiset contrasts are generically $O(\eta)$.  For the normalized fixed-$\beta$ momentum buffer
\begin{equation}
    m^+=\beta m+(1-\beta)g_D(\theta),\qquad \theta^+=\theta-\eta m^+,
\end{equation}
the frozen-parameter endpoint coefficient is
\begin{equation}
    C_1^{\mathrm{mom}}(w)=-\sum_{r=1}^K \bigl(1-\beta^{K-r+1}\bigr)g_{w_r}
\end{equation}
up to the word-independent warmed-buffer term.  These impulse weights depend on position whenever $\beta\ne0$, so equal-multiset contrasts are generically $O(\eta)$.  The unnormalized heavy-ball convention $v^+=\beta v+g_D,\theta^+=\theta-\eta v^+$ gives the same profile divided by $1-\beta$.

AdamW adds de-biasing denominators, coordinate preconditioning, and a second-moment path.  These features change the coefficient and motivate the full-replay comparison, but fixed-clock first-moment memory alone suffices to make the first-order impulse weights position-dependent.  The non-adaptive control compares SGD with fixed-$\beta$ momentum: the two updates differ only by the buffer, and the buffer changes the first-order weights from flat to position-dependent.  To return a buffer to the regular $O(\eta^2)$ exponent, its endpoint impulse profile must flatten with $\eta$.  In the matched-clock control we scale the decay and the normalization together, $\beta=e^{-a\eta}$ with respect to the continuous-time training variable $\tau=\eta k$, so the normalized weight spread is $O(\eta)$.  Scaling the decay alone is not enough for raw heavy-ball momentum, whose $1/(1-\beta)$ gain cancels the flattening.  Table~\ref{tab:clock-taxonomy} summarizes these clock classes.

\begin{table}[t]
\centering
\small
\caption{Clock classes used in the paper.  The class is determined by the reduced first-order impulse profile for equal-multiset words; Muon and Lion are nonsmooth fixed-clock optimizers, included as empirical comparisons outside the closed-form profile.}
\label{tab:clock-taxonomy}
\begin{tabular}{p{0.24\linewidth}p{0.37\linewidth}p{0.25\linewidth}}
\toprule
Optimizer/control & Reduced first-order profile & Order class \\
\midrule
SGD / no buffer & flat over positions & regular, $O(\eta^2)$ \\
fixed-$\beta$ momentum & $1-\beta^{K-r+1}$ & first-order, $O(\eta)$ \\
AdamW fixed clock & finite-clock $W_{r,K}^{(s,\beta)}$ (Eq.~\eqref{eq:W-kernel}) with replay preconditioner & first-order, $O(\eta)$ \\
matched-clock memory & spread shrinks as $O(\eta)$ & regular exponent, $O(\eta^2)$ ($\eta$-dependent clock, not Def.~\ref{def:regular}) \\
Muon / Lion & nonsmooth fixed-clock state & empirical fixed-clock comparisons \\
\bottomrule
\end{tabular}
\end{table}

\begin{assumption}[Nondegenerate replay contrast]\label{assump:nongen}
For the selected equal-multiset words $w,w'$, readout $u$, and measurement state, the replay coefficient $\Lambda_{w,w'}(u)$ in Eq.~\eqref{eq:lifted-contrast} is nonzero.
\end{assumption}

\noindent\textbf{Nondegeneracy.}
The first-order class is conditional on a nonzero replay coefficient.  If $\Lambda_{w,w'}(u)$ vanishes for the selected domains and readout, the leading term may be higher order.  Such vanishing requires an analytic relation among the finite-dimensional local gradient tuples, unless the model family is restricted to that relation.  The theorem is therefore a generic exponent classification, not a guarantee that every AB/BA/readout cell has a nonzero first-order coefficient.

\paragraph{Specialization used below.}
Theorem~\ref{thm:fixed-clock-singularity} is the general lifted-state clock law.  The next sections freeze the coordinate preconditioner $P_0$ at the measurement state to obtain a scalar kernel, a variance formula, and a cycle decomposition.  This frozen-preconditioner step is the main modeling assumption in the closed-form part of the paper: it is not used for the exponent theorem, and it is not treated as a complete symbolic expansion of AdamW.  Full frozen-gradient replay, $\tau_{\mathrm{Adam}}$, keeps the finite-step $m$ and $v$ paths and both clocks, while the conservativity gap $\chi$ reports how far the scalar frozen-$P$ model is from that full replay.  Table~\ref{tab:notation-map} in Section~\ref{sec:transport} consolidates these objects, their roles, and their scope.

\section{AdamW kernel and order-variance floor}
\label{sec:variance-floor}

We now specialize the lifted-state law to AdamW with the coordinate preconditioner frozen at the measurement state, $P_0=(\sqrt{\hat v_0}+\epsilon)^{-1}$.  The first-moment buffer and the bias-correction clock remain finite.  For first-moment decay $\beta$ and current clock $s$, a gradient impulse at position $r$ has endpoint weight
\begin{equation}
\label{eq:W-kernel}
    W^{(s,\beta)}_{r,K}=(1-\beta)\sum_{t=r}^{K}\frac{\beta^{t-r}}{1-\beta^{s+t}}.
\end{equation}
We write $A_T=A_T^{(s,\beta)}=\sum_{q=1}^{T}\beta^q/(1-\beta^{s+q})$ when $s$ and $\beta$ are fixed by context.

\begin{proposition}[Frozen-$P$ first coefficient]\label{prop:first-order-kernel}
For a word $w$,
\begin{equation}
\label{eq:c1-word}
    \theta_K(w)=\theta_0+\eta C_1(w)+O(\eta^2),\qquad
    C_1(w)=-A_KP_0m_0-\sum_{r=1}^{K}W^{(s,\beta)}_{r,K}P_0g_{w_r}.
\end{equation}
The warmed-buffer term is word-independent and cancels from centered order contrasts.  Like that term, $C_1$ is stated at $\lambda=0$: decoupled weight decay adds a word-independent first-order endpoint term $-K\lambda\theta_0$ that also cancels from every centered order contrast.  For balanced blocks $A^{K/2}B^{K/2}$ versus $B^{K/2}A^{K/2}$,
\begin{equation}
\label{eq:balanced-c1}
    \inner{g_E}{C_1(A^{K/2}B^{K/2})-C_1(B^{K/2}A^{K/2})}
    =S_{K,s,\beta}\inner{g_E}{P_0(g_B-g_A)},
\end{equation}
where $S_{K,s,\beta}=\sum_{r\le K/2}W_{r,K}^{(s,\beta)}-\sum_{r>K/2}W_{r,K}^{(s,\beta)}$.  We also use the dimensionless coefficient
\begin{equation}
\label{eq:c-normalized}
    c_{K,s,\beta}=\frac{(1-\beta^{s+1})S_{K,s,\beta}}{K(1-\beta)}.
\end{equation}
\end{proposition}

For AdamW, the fixed-clock mechanism is visible in the kernel: non-flat finite-clock weights turn an equal-multiset contrast into an order-$\eta$ displacement.  At $K=16,s=100,\beta_1=0.9$, the endpoint impulse weight falls from $0.815$ to $0.100$ across the window.

The same kernel gives the distribution of outcomes over uncontrolled orderings.  Let $\phi_i=\inner{g_E}{P_0g_i}$ for the $K$ blocks in a fixed multiset, let $\sigma_\phi^2$ be their finite-population variance, and define
\begin{equation}
\label{eq:VK}
    V_K=\sum_{r=1}^K\left(W^{(s,\beta)}_{r,K}-\bar W\right)^2,
    \qquad
    \bar W=K^{-1}\sum_r W^{(s,\beta)}_{r,K}.
\end{equation}

\begin{theorem}[Order-fluctuation expansion and asymptotic floor]
\label{thm:variance-floor}
For a uniformly random permutation $w$ of a fixed multiset, the frozen-preconditioner expansion satisfies
\begin{equation}
\label{eq:variance-floor}
    \operatorname{Var}_w\bigl[\inner{g_E}{\theta_K(w)}\bigr]
    =\eta^2\frac{K}{K-1}\sigma_\phi^2V_K+2\eta^3\operatorname{Cov}_w(L,Q)+O(\eta^4),
\end{equation}
where $L$ and $Q$ are the first- and second-order word coefficients.  Regular optimizers have flat first-order weights, so $V_K=0$ and the leading order-induced variance is $O(\eta^4)$.
\end{theorem}

\begin{proof}[Proof sketch]
After centering, the first-order orbit statistic is $Y(w)=-\eta\sum_r W_{r,K}^{(s,\beta)}\phi_{w_r}$, the $\eta$-scaled, warmed-buffer-centered form of the first-order coefficient $L$ in Eq.~\eqref{eq:variance-floor}.  Sampling a fixed multiset without replacement gives $\operatorname{Var}(\phi_{w_r})=\sigma_\phi^2$ and $\operatorname{Cov}(\phi_{w_r},\phi_{w_t})=-\sigma_\phi^2/(K-1)$ for $r\ne t$.  The variance of this weighted linear statistic gives the factor $\frac{K}{K-1}\sum_r(W_r-\bar W)^2$ \citep{waldwolfowitz1944,hoeffding1951}; the $\eta^3$ term is the first interaction with transport curvature.  The full proof and the per-class $V_K$ computations are in Appendix~\ref{app:variance-floor-proof}.
\end{proof}

The leading variance term is determined by the step scale, the measured heterogeneity $\sigma_\phi^2$, and the optimizer-memory contrast $V_K$.  At a fixed measurement state, that term is an asymptotic floor for the order component, not a pointwise finite-$\eta$ bound or a bound against other stochastic sources.  Writing the order variance as $\eta^2\mathcal{A}_K+\eta^3\mathcal{B}_K+O(\eta^4)$ with $\mathcal{A}_K=\frac{K}{K-1}\sigma_\phi^2 V_K$ and $\mathcal{B}_K$ the $\eta^3$ transport covariance, the perturbative-transport condition $|\eta^3\mathcal{B}_K|\le\gamma\,\eta^2\mathcal{A}_K$ with $\gamma<1$ for all $\eta\le\eta_0$ gives $\operatorname{Var}_w\ge(1-\gamma)\,\eta^2\mathcal{A}_K$, an $\eta^2$ lower bound that regular optimizers lack (Appendix~\ref{app:floor-scope}); at finite $\eta$, the $\eta^3$ covariance can be negative.  Across a learning-rate sweep the warmed state can change $\sigma_\phi^2$, so empirical variance exponents should be read as class separation plus measured-input decomposition, not as exact parameter-free exponent fits.  Full frozen-gradient AdamW still has a first-order replay by Theorem~\ref{thm:fixed-clock-singularity}, but its $v$ path need not reduce to the scalar statistic in Eq.~\eqref{eq:variance-floor}; for that reason we compare the frozen-$P$ specialization with full replay and $\chi$.

\paragraph{Matched-clock restoration.}
When the memory clock is matched to the continuous-time training variable $\tau=\eta k$, $\beta=e^{-a\eta}$ and $s_0=T_c/\eta$, the normalized impulse weights flatten to an $O(\eta)$ spread.  With saturated de-biasing this makes $V_K=\Theta(\eta^2)$ and returns the order variance to the regular $O(\eta^4)$ class.  The matched clock is not a regular family in the sense of Definition~\ref{def:regular}, since $\beta$ and $s_0$ depend on $\eta$; it reaches the regular exponent only through the flattened reduced impulse profile.  The matched-clock arm co-scales decay and normalization; the separate SGD versus fixed-$\beta$ momentum control isolates fixed-clock buffering from coordinate adaptivity.

\paragraph{Shuffle-seed budget.}
Using one shuffle seed makes a comparison reproducible, but it does not estimate the order component of uncertainty.  If order noise has standard deviation $\sigma_{\mathrm{ord}}$ and two configurations differ by a true gap $\Delta$, the per-side seed count for a level-$\alpha$ comparison with power $1-\beta_{\mathrm{pow}}$ is
\begin{equation}
\label{eq:seed-count}
    n^*\ge \frac{2\sigma_{\mathrm{ord}}^2(z_{1-\alpha/2}+z_{1-\beta_{\mathrm{pow}}})^2}{\Delta^2}
    \approx 15.7\left(\frac{\sigma_{\mathrm{ord}}}{\Delta}\right)^2
\end{equation}
for $\alpha=0.05$ and 80\% power.  The seed-count calculation estimates $\sigma_{\mathrm{ord}}$ from Eq.~\eqref{eq:variance-floor} or from a small ordering pilot and asks whether the seed budget resolves the reported gap.

\paragraph{Held-out-NLL calibration.}
A held-out-NLL calibration puts the scale in a concrete model-selection comparison.  On Llama-3.2-1B LoRA SFT, we compare a 50/50 math--code mix with a 70/30 mix for 512 AdamW steps, using 25 shuffle seeds per side and a fixed mixed held-out NLL probe.  The larger learning rates satisfied the degeneracy criteria without exclusions:
\begin{center}
\small
\begin{tabular}{lcccc}
\toprule
$\eta$ & mean gap A--B & pooled $\sigma_{\mathrm{ord}}$ & $|\Delta|/\sigma_{\mathrm{ord}}$ & sign-change rate \\
\midrule
$5\times10^{-5}$ & $+0.00159$ & $0.00111$ & $1.43$ & 16.5\% \\
$1\times10^{-4}$ & $+0.00083$ & $0.00168$ & $0.49$ & 33.0\% \\
$2\times10^{-4}$ & $-0.00127$ & $0.00422$ & $0.30$ & 43.7\% \\
\bottomrule
\end{tabular}
\end{center}
The absolute NLL gaps are small, and the mean gap changes sign across the sweep.  The held-out-NLL table is a power calibration: when the configuration gap is comparable to the order-noise floor, a single-shuffle estimate can change the sign of the measured comparison.  The table calibrates the seed budget rather than choosing between these mixes.

\section{Local diagnostics and \texorpdfstring{frozen-$P$}{frozen-P} scope}
\label{sec:transport}
\label{sec:sorting-curl}

The first-order coefficient is evaluated at the measurement state, whereas live training evaluates gradients along the perturbed path.  With the local expansion $g_D(\theta)=g_D+H_D(\theta-\theta_0)+O(\|\theta-\theta_0\|^2)$, substituting the first-order prefix trajectory gives the following coefficient.

\begin{proposition}[Frozen-$P$ transport coefficient]\label{prop:second-order-transport}
The endpoint has expansion
\begin{equation}
\label{eq:second-order-word}
    \theta_K(w)=\theta_0+\eta C_1(w)+\eta^2C_2(w)+O(\eta^3),
\end{equation}
\begin{equation}
    C_2(w)=\sum_{t=1}^{K}W_{t,K}P_0H_{w_t}\left(\sum_{r<t}W_{r,t-1}P_0g_{w_r}+A_{t-1}P_0m_0\right).
\end{equation}
\end{proposition}
The $m_0$ term is needed: it cancels from $C_1$ contrasts but enters $C_2$ through $H_{w_t}P_0m_0$, yielding $M_{D,0}$ below.  The displayed coefficient omits weight decay; Appendix~\ref{app:transport-weight-decay} gives the addendum.  This distinction matters because $\rho$ compares the live curvature correction $\Delta^{(1)}$ with the full frozen-gradient replay $\tau_{\mathrm{Adam}}$.

\paragraph{Trust-region ratio.}
The full frozen-gradient AdamW replay is denoted $\tau_{\mathrm{Adam}}$.  It includes finite-step $m$ and $v$, both clocks, $\epsilon$, and weight decay.  The ratio
\begin{equation}
\label{eq:rho}
    \rho=\left|\frac{\Delta^{(1)}}{\tau_{\mathrm{Adam}}}\right|
\end{equation}
compares curvature transport to the first-order optimizer-state response.  The ratio is a trust-region diagnostic: $\rho\ll1$ means first-order replay dominates, $\rho\sim1$ means transport is comparable to replay, and $\rho\gg1$ marks cells outside the local first-order regime.  When $|\tau_{\mathrm{Adam}}|$ is near zero, the ratio is ill conditioned: the cell marks a cancellation of the first-order replay coefficient, not a stable local-validity estimate.  The local order diagnostic of Algorithm~\ref{alg:diagnostic-app} (Appendix~\ref{app:protocol}) treats these as denominator-singular cases.  The correction $\tau_{\mathrm{Adam}}+\Delta^{(1)}$ is therefore a local approximation valid in the perturbative ($\rho<1$) regime.  The superscript in $\Delta^{(1)}$ marks the leading transport correction to replay, which enters at order $\eta^2$; it does not denote the order-$\eta$ first-order channel.

\paragraph{Finite-window curl and sortability.}
The frozen-$P$ specialization also gives a condition for when pairwise order preferences behave like differences of scalar domain values.  We use \emph{curl} in the finite-window, graph-theoretic sense: the signed sum of pairwise edge scores around a directed domain triangle.  If an edge score is a scalar potential difference, this triangle sum telescopes to zero; nonzero curl means pairwise preferences can form directed 3-cycles.  With $\phi_i=\inner{g_E}{P_0g_i}$, the first-order edge for domains $i,j$ is proportional to $\phi_j-\phi_i$, so its curl is zero.  The second-order edge for balanced blocks $i^hj^h$ versus $j^hi^h$ decomposes as
\begin{equation}
\label{eq:edge-decomp-main}
    e^{(2)}_{ij}=a_K(M_{ii}-M_{jj})+b_K(M_{ji}-M_{ij})+c_K(M_{i0}-M_{j0}),
\end{equation}
where $M_{D,q}=\inner{g_E}{P_0H_DP_0g_q}$ and $M_{D,0}=\inner{g_E}{P_0H_DP_0m_0}$.  The first and third terms are conservative potential differences.  The only term that can create nonzero triangle curl is
\begin{equation}
\label{eq:curl-main}
    B_{ij}=b_K(M_{ji}-M_{ij})=b_K\inner{P_0g_E}{H_jP_0g_i-H_iP_0g_j}.
\end{equation}
Thus pairwise order stays sortable by a scalar potential while the curl-carrying bracket term stays below the conservative (scalar-potential) part, in the measured perturbative regime $\rho_{\mathrm{curl}}<1$ (Section~\ref{sec:experiments}; the full edge decomposition is derived in Appendix~\ref{app:edge-decomp}).  We use
\begin{equation}
\label{eq:rhocurl}
    \rho_{\mathrm{curl}}(i,j)=\frac{|B_{ij}|}{|\Pi_{ij}|+\epsilon}
\end{equation}
as the edge-level radius variable, where $\Pi_{ij}$ is the conservative part plus the frozen-$P$ first-order edge and $\epsilon$ is a small fixed numerical stabilizer.

\paragraph{Bridge to full AdamW.}
Full AdamW can depart from this scalar picture through the $v$ path of its preconditioner.  Appendix~\ref{app:fullv-tangent} gives a full-$v$ AdamW correction recursion, computed by a backward-sensitivity (costate) sweep (Appendix~\ref{app:costate}): it reduces to the displayed frozen-$P$ coefficient when the derivative of the preconditioner map is dropped.  We quantify the scalar approximation by the conservativity gap
\begin{equation}
\label{eq:chi}
    \chi=\frac{|\tau_{\mathrm{Adam}}-\tau_P|}{|\tau_{\mathrm{Adam}}|+\epsilon}.
\end{equation}
Here $\tau_P$ is the scalar frozen-$P$ replay coefficient, the frozen-$P$ analog of $\tau_{\mathrm{Adam}}$ with the preconditioner map held fixed.  Low $\chi$ means the scalar frozen-$P$ coefficient is close to full replay; high $\chi$ means first-order replay may still be predictable but not scalar-potential conservative.  Low $\rho_{\mathrm{curl}}$ means pairwise preferences are locally scalar-sortable; when $\rho_{\mathrm{curl}}$ approaches one, bracket transport can create directed cycles.  A directed cycle, when it occurs, is a coefficient-level intransitive order preference that no scalar potential can represent: the unstabilized edge condition $|B_{ij}|<|\Pi_{ij}|$ at every edge is sufficient for local scalar sortability, and $\rho_{\mathrm{curl}}<1$ is the corresponding stabilized numerical radius; outside it scalar ranking is not guaranteed and directed cycles are empirically observed (Section~\ref{sec:experiments}).

The three radii are non-redundant axes: $\rho$ bounds the local first-order approximation, $\chi$ the scalar-versus-full-replay gap, and $\rho_{\mathrm{curl}}$ the scalar-sortability of pairwise order, so a cell can lie inside one radius and outside another.

\begin{table}[t]
\centering
\small
\caption{Dependency map for the order-coefficient chain $\Lambda\!\to\!\tau_{\mathrm{Adam}}\!\to\!\tau_P/W_{r,K}/V_K\!\to\!\rho/\chi/\rho_{\mathrm{curl}}$.  Each row records an object's role and the scope under which it is stated; no entry is a new result.}
\label{tab:notation-map}
\begin{tabularx}{\linewidth}{>{\raggedright\arraybackslash}p{0.12\linewidth}>{\raggedright\arraybackslash}X>{\raggedright\arraybackslash}p{0.29\linewidth}}
\toprule
Object & Role & Scope \\
\midrule
$\Lambda_{w,w'}(u)$ & First-order replay coefficient: $\inner{u}{\theta_K(w)-\theta_K(w')}=\eta\,\Lambda_{w,w'}(u)+O(\eta^2)$, nonzero gives a $\Theta(\eta)$ order contrast (Eq.~\eqref{eq:lifted-contrast}) & General lifted-state law; finite window $K$ at a fixed measurement state \\
$\tau_{\mathrm{Adam}}$ & Full frozen-gradient AdamW replay: finite-step $m,v$, both clocks, $\epsilon$, weight decay & Reference replay for $\rho$ and $\chi$ \\
$\tau_P$ & Scalar frozen-$P$ replay: the analog of $\tau_{\mathrm{Adam}}$ with the preconditioner map held fixed & Frozen-$P$ specialization (\S\ref{sec:transport}) \\
$W^{(s,\beta)}_{r,K}$ & Finite-clock impulse weight; non-flat weights make the equal-multiset contrast order $\eta$ (Eq.~\eqref{eq:W-kernel}) & Frozen-$P$ AdamW kernel \\
$V_K$ & Optimizer-memory contrast in the leading order variance $\tfrac{K}{K-1}\sigma_\phi^2 V_K$ (Eqs.~\eqref{eq:VK},~\eqref{eq:variance-floor}); $V_K=0$ for regular optimizers & Asymptotic floor at a fixed measurement state, not a pointwise finite-$\eta$ bound \\
$\rho$ & Trust-region ratio $|\Delta^{(1)}/\tau_{\mathrm{Adam}}|$: curvature transport versus first-order replay (Eq.~\eqref{eq:rho}) & Local approximation valid in the perturbative $\rho<1$ regime \\
$\chi$ & Conservativity gap $|\tau_{\mathrm{Adam}}-\tau_P|/(|\tau_{\mathrm{Adam}}|+\epsilon)$: scalar frozen-$P$ versus full replay (Eq.~\eqref{eq:chi}) & Low $\chi$: scalar close to full replay; high $\chi$: predictable but not scalar-conservative \\
$\rho_{\mathrm{curl}}$ & Edge radius $|B_{ij}|/(|\Pi_{ij}|+\epsilon)$: scalar-sortability of pairwise order (Eq.~\eqref{eq:rhocurl}) & Locally sortable when $|B_{ij}|<|\Pi_{ij}|$ at every edge; $\rho_{\mathrm{curl}}<1$ is the stabilized radius, cycles possible outside \\
\bottomrule
\end{tabularx}
\end{table}

\section{Empirical evidence}
\label{sec:experiments}

We test the clock law in mean AB/BA contrasts, ordering-orbit distributions, local validity, and a held-out-NLL shuffle-seed calibration.  All measurements are local LoRA fine-tuning cells.  Table~\ref{tab:evidence} summarizes the main measurements; Appendix~B--C and the artifact give protocol details, confidence intervals, row-level records, and interval-overlap cases.  The supplementary ZIP is organized as one directory per experiment---reproduction scripts, row-level records, and protocol locks, with no author-identifying metadata---and its README maps each directory to the figures and tables it reproduces (for example, \texttt{mean\_exponent} for the exponent split, \texttt{variance\_slope} for the order-variance floor and fixed-state replay, and \texttt{c1\_decision\_flip} for the held-out-NLL calibration).

\begin{table}[t]
\centering
\small
\caption{Main empirical measurements.  The table gives one anchor result per result type; full intervals are in Appendix~C and protocol details in Appendix~B, and the headline experiment rows are reproduced by the artifact.}
\label{tab:evidence}
\begin{tabular}{p{0.30\linewidth}p{0.55\linewidth}}
\toprule
Measurement & Anchor result \\
\midrule
Mean exponent split & SGD/matched-clock near $2$; AdamW/Lion near $1$ on two models and three domain pairs \\
Buffer mechanism & fixed-$\beta$ momentum slope $0.998$ versus buffer-free SGD $1.988$ \\
Variance floor & warmed AdamW variance slope $1.20$ versus matched/SGD $3.94/3.83$; fixed-state replay gives AdamW/fixed-$\beta$/SGD slopes $1.83/2.00/4.00$ \\
Ordering-orbit structure & fixed-clock AdamW ordering correlations $0.75$--$0.98$; matched-clock cells $\approx1$ \\
Additional optimizer & Muon variance slope $1.56$ $[1.31,1.82]$; Lion variance test is inconclusive \\
Local validity & $\rho<1$ residual Pearson $0.80$; zero directed 3-cycles (coefficient level) when every edge has $\rho_{\mathrm{curl}}<1$ \\
Shuffle-seed calibration & held-out-NLL mix comparison: $33$--$44\%$ sign changes when $|\Delta|/\sigma_{\mathrm{ord}}<1$ \\
\bottomrule
\end{tabular}
\end{table}

\paragraph{Exponent and mechanism.}
Figure~\ref{fig:eta-slope} is the mean-level test of Theorem~\ref{thm:fixed-clock-singularity}.  On Pythia-1B \citep{biderman2023pythia} and Llama-3.2-1B \citep{grattafiori2024llama3}, SGD and matched-clock memory scale near $\eta^2$, while fixed-clock AdamW and Lion scale near $\eta$.  A separate Pythia-1B control separates buffering from adaptivity: fixed-$\beta$ momentum has median slope $0.998$ over four domain pairs and three seeds, while buffer-free SGD has slope $1.988$ and matched-clock momentum has $1.960$.  Thus adding only a fixed-clock buffer moves the mean order-effect exponent from two to one.

\paragraph{Variance floor.}
A direct fixed-state replay isolates the class separation: warming one Pythia-1B cell, forking $(\theta,m,v,s)$, and replaying every $\eta$ from the same fork with bitwise-deterministic steps gives variance slopes $1.83/2.00/4.00$ for AdamW, fixed-$\beta$ momentum, and SGD (Appendix~\ref{app:fixed-state-replay}).  The warmed multi-cell sweep (Figure~\ref{fig:variance-slopes}) agrees across five learning rates and four model--pair series: the Theil--Sen slope of $\log\operatorname{Var}_w$ versus $\log\eta$ is $1.20$ for AdamW, versus $3.94$ for matched-clock momentum and $3.83$ for SGD, with non-overlapping bootstrap intervals ($[0.80,1.56]$, $[3.85,4.03]$, $[3.58,4.14]$; the AdamW interval is wide but does not overlap the regular arms).  The warmed AdamW slope sits below the fixed-$\sigma_\phi$ exponent because the warmed state changes the heterogeneity input $\sigma_\phi^2(\eta)$; dividing by the measured $\sigma_\phi^2(\eta)$ raises it to $1.86$.

\paragraph{Ordering-orbit structure.}
Ordering-by-ordering comparisons test the kernel's ordering/rank structure: across eight Pythia/Llama cells with 64 deterministic orderings each, the fixed-clock AdamW ordering-level correlations are $0.75$--$0.98$, while matched-clock cells return correlation essentially one.  The per-cell amplitude slopes are $0.62$--$0.80$: the kernel captures rank, not amplitude, a deficit due to the measured $\eta^3$ cross-term of Theorem~\ref{thm:variance-floor}, reported in Appendix~C and compared with the full-$v$ replay gap $\chi$.

\paragraph{Additional optimizers.}
These checks follow update structure rather than the AdamW derivation.  Muon's orthogonalized momentum is consumed on an $\eta$-independent clock; because its update is nonsmooth, its fixed-clock placement is empirical, outside the closed-form AdamW coefficient, with variance slope $1.56$ $[1.31,1.82]$.  Lion's mean contrast is fixed-clock, but its variance slope $2.87$ $[2.31,3.94]$ is inconclusive.  The closed-form floor remains the frozen-$P$ AdamW specialization; the broader result is the fixed-clock state exponent law.

\begin{figure}[t]
\centering
\includegraphics[width=0.78\linewidth]{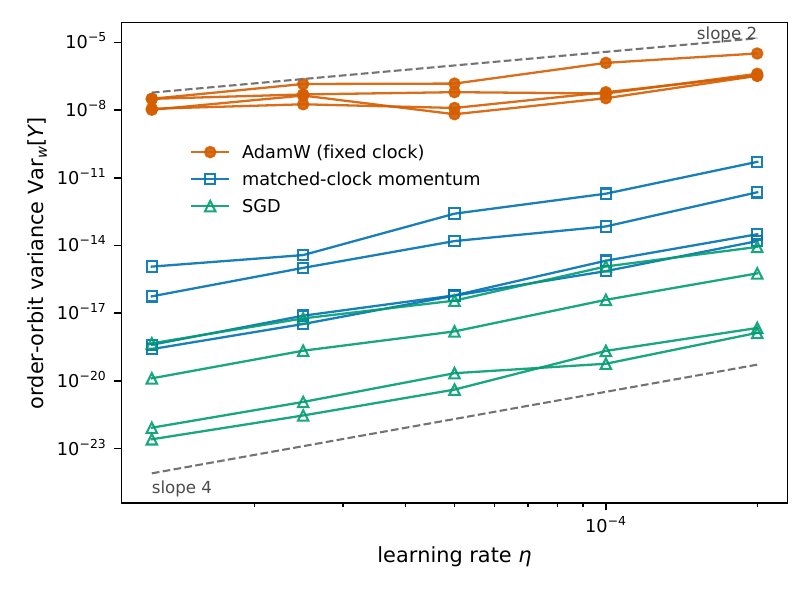}
\caption{Variance-level exponent split.  Fixed-clock AdamW order-orbit variance is several orders of magnitude larger than SGD and matched-clock memory and follows the shallow scaling expected from a first-order channel.  The regular controls follow the $O(\eta^4)$ law.  A direct fixed-state replay (one forked warmed state replayed across $\eta$) gives variance slopes $1.83$/$2.00$/$4.00$ for AdamW/fixed-$\beta$/SGD; the warmed sweep plotted here additionally varies the measured heterogeneity input $\sigma_\phi^2(\eta)$, lowering the plotted AdamW slope to $1.20$ before the measured-input correction.}
\label{fig:variance-slopes}
\end{figure}

\paragraph{Validity and sortability.}
The 192-row diagnostic grid spans Qwen \citep{yang2025qwen3}, Pythia, Llama, and Gemma \citep{gemma2024gemma2} local cells and tests whether local coefficients predict live AB/BA outcomes.  In the perturbative regime $\rho<1$, the frozen-$P$ transport coefficient predicts the live residual $\Delta_{\mathrm{live}}-\tau_{\mathrm{Adam}}$ with Pearson $0.80$; at high $\rho$, correlation becomes uninformative.  The very-low-$\rho$ bin is noise-limited because the residual correction is tiny once frozen replay already dominates.  As a predictor of approximation error, $\rho$ beats block length, gradient norms, curvature magnitude, denominator magnitude, and numerator magnitude, with Spearman $0.63$ and AUC $0.79$.  This comparison uses the full per-row diagnostic grid, released as the \texttt{baseline\_horserace} bundle in the artifact; Appendix~C reports the baseline table and uncertainty, and the artifact reproduces both the $\rho<1$ residual-validation rows and this $\rho$-versus-baselines comparison.  Thus the diagnostic has the expected ratio-test behavior: predictive accuracy inside the regime and loss of accuracy outside it.

The fixed-target triple experiment uses seven Pile domains on Llama-3.2-1B with a common evaluation target for every edge in a triangle.  The conservative component never cycles.  Across 700 coefficient cells, the full frozen-$P$ edge field has zero directed 3-cycles in all 286 cells with every edge satisfying $\rho_{\mathrm{curl}}<1$, and cycles appear above the radius.  The zero-cycle result supports the scalar-sortability diagnostic at coefficient level; live validation of cyclic cells was not included in this measurement.  We separately compare scalar frozen-$P$ with full replay: across the 192-row grid, median $\chi$ is $0.044$, 78\% of rows have $\chi<0.25$, and 90\% have $\chi<0.5$.

\section{Related work}
\label{sec:related}

\paragraph{Order expansions and operator brackets.}
Commutators are classical in geometric control and numerical integration \citep{sussmann1973orbits,jurdjevic1997geometric,bullo2004geometric}.  In learning, backward-error, multi-domain, and transfer-order work uses memoryless gradient brackets such as $\inner{g_E}{H_Bg_A-H_Ag_B}$ to reason about domain order \citep{dherin2023implicit,rukhovich2025commute,sweeney2026liebracket}.  Concurrent work frames training order as an information channel and a dominant share of the cumulative gradient \citep{ledoux2026order}; that per-step coherent-order quantity is compatible with our equal-multiset endpoint contrast, since a large per-step gradient share need not imply a large equal-multiset endpoint effect.  Closest methodologically, a recent splitting-method analysis treats SGD and momentum as splitting integrators and random reshuffling as a randomised splitting order, using backward-error analysis and Lie--Trotter versus Strang composition to derive step-size order \citep{shaw2025splitting}.  In that setting, symmetric minibatching with momentum \emph{lowers} the stochastic-gradient bias order (from $O(h^2)$ to $O(h^4)$), whereas our fixed-clock analysis moves the equal-multiset order exponent in the opposite direction, from two to one; the step-size $h$ there governs discretization bias, while $\eta$ here governs reordering sensitivity.  A related backward-error analysis isolates an order-dependent within-epoch correction carrying a finite-sum factor $(N-B)/(N-1)$ analogous to our $K/(K-1)$ \citep{smith2021origin}, but in the \emph{mean} modified loss for momentum-free SGD, and without-replacement trajectories admit a comparable memoryless decomposition into a with-replacement step plus a regularizer \citep{beneventano2023trajectories}; these corrections are memoryless and second order.  Memory-based backward-error analyses extend the line to momentum and AdamW.  The current-iterate collapse of \citet{cattaneo2025memory} shares our frozen-state collapse and even our bias-correction-clock coefficients, but it averages over minibatch orderings into an order-independent term.  The Adam expansion of \citet{cattaneo2024implicit} gives single-trajectory implicit bias as a perturbed one-norm regularizer; its minibatch corrections are order-sensitive in form but are not resolved into an equal-multiset contrast.  Here the objects are instead the run-to-run order \emph{variance} and its dependence on the optimizer buffer, so our question is how the bracket picture changes once optimizer state enters the local map: fixed-clock state moves equal-multiset order from the memoryless $O(\eta^2)$ bracket regime to $O(\eta)$, with local-validity ratios for the first-order calculation.

\paragraph{Curricula, reshuffling, and empirical order effects.}
Curriculum learning studies how examples or tasks should be ordered \citep{bengio2009curriculum,kumar2010selfpaced,graves2017automated,soviany2022curriculum}.  LLM post-training work often emphasizes instruction data, mixture design, or preference optimization \citep{ouyang2022training,rafailov2023dpo}; we study the optimizer-state order channel inside such local fine-tuning comparisons, rather than an ordering or mixture-selection policy.  Empirically, fine-tuning outcomes vary with data seed and order \citep{dodge2020finetuning,chen2024order}, training-order recency is linearly encoded in trained activations \citep{krasheninnikov2025fresh}, and optimizer-state memory has been measured to carry this order sensitivity---amplified by momentum and collapsing when the optimizer state is reset \citep{sevetlidis2026processtensor}.  We add a local fixed-clock exponent and order-variance scale for that optimizer-state path dependence.  A parallel optimization literature studies random reshuffling and without-replacement SGD---convergence-rate and optimal-permutation analyses \citep{gurbuzbalaban2021reshuffling,haochen2019random,ahn2020shuffling,shamir2016without,rajput2020closing,lu2022grab} and adversarial data-ordering attacks \citep{shumailov2021ordering}.  The reshuffling literature concerns convergence rates or beneficial/harmful example orders for SGD-type updates over a full pass; we instead study local equal-multiset order contrasts and how fixed-clock optimizer state shifts the small-$\eta$ exponent.  The closest variance analysis there computes the stationary iterate variance from gradient noise \citep{domingoenrich2022variance}, not the order-orbit variance at fixed data that our floor describes.

\paragraph{Optimizer state, adaptivity, and continuous-time limits.}
Adam/AdamW and related adaptive optimizers are standard in fine-tuning \citep{kingma2015adam,loshchilov2019decoupled,duchi2011adaptive,reddi2018convergence,shazeer2018adafactor,gupta2018shampoo}, as are sign- and orthogonalization-based updates such as Lion \citep{chen2023symbolic} and Muon \citep{jordan2024muon}.  Prior analyses focus on convergence, implicit bias, scaling rules, and SDE or continuous-time approximations \citep{mandt2017sgd,li2017stochastic,malladi2022sde,li2025sharpness}.  Here we study a local algebraic question: with a fixed optimizer state and clock, what changes when two equal-multiset data blocks are swapped?  The mechanism is broader than adaptivity; any non-flat fixed-clock replay can change the exponent class.  Continuous-time limits are complementary precisely because they \emph{match the clock} by construction: the SDE derivations co-scale the moment decay with the step size (the adaptive square-root scaling rule sets $1-\beta_2=\Theta(\eta^2)$, and the momentum drift forces $1-\beta_1=\Theta(\eta)$) and carry the bias-correction counter as the continuous-time variable $\tau=\eta k$ \citep{li2017stochastic,malladi2022sde}.  In our notation the normalized impulse weights then flatten as $\eta\to0$ and the first-order contrast vanishes, so the limit sits in the matched-clock regime rather than a short finite-clock AdamW window.  SDE and continuous-time limits therefore correspond to the matched-clock regime; the fixed-clock channel is a finite-clock regime that disappears when the optimizer clock is scaled with $\eta$.

\paragraph{Fine-tuning geometry and attribution.}
Hessian-vector products enter only as local diagnostics, following the Pearlmutter trick \citep{pearlmutter1994fast}, to estimate contractions such as $\inner{g_E}{P_0H_DP_0g_j}$ that control transport and pairwise-order cycles.  The LoRA setting \citep{hu2022lora} is related to influence and kernel views of fine-tuning \citep{malladi2023kernel}, but our focus is optimizer-state path dependence rather than per-example influence.  Scaling laws relate loss to parameters, data, and compute \citep{kaplan2020scaling,hoffmann2022training}; Pythia, Qwen, Gemma, and Llama provide model families \citep{biderman2023pythia,yang2025qwen3,gemma2024gemma2,grattafiori2024llama3}, while our local variables are optimizer-state radius and clock class.

\section{Discussion}
\label{sec:discussion}

Fixed-clock optimizer state changes the local role of order: a regular optimizer's first-order response to an equal multiset is a flat sum, while fixed-clock state can make replay position-dependent.  In this sense the memoryless bracket picture is the commutative limit: it applies when first-order impulse weights are flat, and becomes incomplete when fixed-clock state makes them non-flat; the matched-clock control co-scales the memory and normalization clocks so that this flat-profile limit returns.

\paragraph{Scope.}
The scope is local to a finite measurement window.  State should be recomputed before applying the expansion elsewhere in training.  The closed-form scalar potential, the leading variance coefficient, and the cycle theorem are exact for the frozen-preconditioner specialization; the order-variance \emph{floor} interpretation of that coefficient is asymptotic, not a pointwise finite-$\eta$ bound.  Full frozen-gradient AdamW is covered by Theorem~\ref{thm:fixed-clock-singularity} and by replay measurements, but the $v$-dependent preconditioner can add first-order effects outside the scalar-potential form; this is why $\chi$ is reported.  Most diagnostic rows are single-seed local cells, so the strongest conclusions are coefficient structure, class separation, and validity ranking, not parameter-count scaling laws.

\paragraph{Finite-horizon accounting.}
If the expansion is recomputed window by window, a run of $N_{\mathrm{w}}$ windows has the first-variation form
\begin{equation}
\label{eq:finite-horizon}
\Delta_{\mathrm{end}}=\eta\sum_{p=1}^{N_{\mathrm{w}}}G_p\zeta_p+O(N_{\mathrm{w}}\eta^2),
\end{equation}
where $\zeta_p$ is the window's frozen-state order coefficient and $G_p$ transports that displacement to the endpoint readout.  Independently reshuffled windows give random-walk scale $O(\eta\sqrt{N_{\mathrm{w}}})$ under mean-zero weak dependence; coherently ordered windows can drift as $O(\eta N_{\mathrm{w}})$ while the local regime remains valid and the suffix maps are contractive or geometrically mixing (mere non-expansiveness is not enough).  Matched-clock memory replaces the leading $\eta$ by $\eta^2$.  Proposition~\ref{prop:finite-horizon} formalizes this averaging-versus-accumulation accounting.

\paragraph{Practical use.}
The shuffle-seed budget is an error-bar calculation for the order component: gaps small relative to $\sigma_{\mathrm{ord}}$ can be underpowered, while large gaps are unlikely to flip.

\bibliography{references}
\bibliographystyle{tmlr}

\clearpage
\appendix
\section{Proofs and Derivations}
\label{app:proofs}

\subsection{Regular commutator baseline}
Let $F_D^\eta=I+\eta X_D+\eta^2Y_D+O(\eta^3)$, uniformly on a neighborhood of the measurement state.  For a vector field $X$, write $\mathrm{D}X(z)[h]$ for its derivative at $z$ applied to direction $h$.  Expanding the two compositions gives
\begin{align}
F_B^\eta(F_A^\eta(z))
&=z+\eta(X_A+X_B)(z)
  +\eta^2\{Y_A+Y_B+\mathrm{D}X_B[X_A]\}(z)+O(\eta^3),\\
F_A^\eta(F_B^\eta(z))
&=z+\eta(X_A+X_B)(z)
  +\eta^2\{Y_A+Y_B+\mathrm{D}X_A[X_B]\}(z)+O(\eta^3).
\end{align}
With the bracket convention $[X_B,X_A]:=\mathrm{D}X_B[X_A]-\mathrm{D}X_A[X_B]$, subtracting yields
\begin{equation}
    \eta^2\{\mathrm{D}X_B[X_A]-\mathrm{D}X_A[X_B]\}(z)+O(\eta^3)
    =\eta^2[X_B,X_A](z)+O(\eta^3).
\end{equation}

\subsection{Fixed-clock singularity and nondegeneracy}
\label{app:fixed-clock-nondegen}
This subsection gives the proof of Theorem~\ref{thm:fixed-clock-singularity}.  Fix a finite word length $K$ and assume the maps in Eq.~\eqref{eq:lifted-map} are $C^1$ and uniformly bounded on a compact neighborhood of the measurement state.  Let $C$ denote a generic constant independent of $\eta$ and of the word.

First, $\|\theta_t-\theta_0\|\le C\,t\eta$ for all $t\le K$, because every parameter increment is $O(\eta)$.  Let $e_t=\|\xi_t(w)-\xi_t^0(w)\|$.  Lipschitz continuity of $G_D$ (constant $L_G$) gives
\begin{equation}
    e_t\le L_G e_{t-1} + L_G\|\theta_{t-1}-\theta_0\| \le L_G e_{t-1}+C\eta,
\end{equation}
so $e_t=O(\eta)$ for fixed $K$ by induction.  Now write the endpoint update as
\begin{equation}
    \theta_K(w)-\theta_0
    =\eta\sum_{t=1}^{K}U_{w_t}(\xi_t(w),\theta_{t-1}(w))+O(\eta^2).
\end{equation}
Replacing $(\xi_t(w),\theta_{t-1}(w))$ by $(\xi^0_t(w),\theta_0)$ changes each summand by $O(\eta)$ and therefore the endpoint by $O(\eta^2)$, proving Eq.~\eqref{eq:lifted-expansion}.  For AdamW, the Lipschitz constant of $U_D$ used in this replacement is finite precisely when the warmed preconditioner coordinates satisfy $\hat v\ge\delta>0$ (the $\epsilon$ in $\sqrt{\hat v}+\epsilon$ bounds the value of the map, not its derivative); where that fails, the $C^1$ argument does not apply, and the theorem applies only on the smooth region of the finite path.  Subtracting two word expansions gives Eq.~\eqref{eq:lifted-contrast}.  Assumption~\ref{assump:nongen} then gives a $\Theta(\eta)$ readout.

The older linear-memory example is the special case $\xi=y$, $G_D(y,\theta)=\mathcal{M}y+\mathcal{B}g_D(\theta)$, and $U_D(y,\theta)=-\mathcal{C}_ty$.  Setting $\eta=0$ in the parameter equation while retaining the memory recursion gives
\begin{equation}
    y_t^{0}(w)=\mathcal{M}^t y_0+\sum_{r\le t}\mathcal{M}^{t-r}\mathcal{B} g_{w_r}(\theta_0),
\end{equation}
with first parameter coefficient
\begin{equation}
    \theta_K^{[1]}(w)=-\sum_{t=1}^{K} \mathcal{C}_{t-1} y_t^{0}(w).
\end{equation}
For two equal-multiset words, common terms depending only on $y_0$ cancel, leaving a linear combination of the frozen domain gradients.  Vanishing of this coefficient is a proper linear/algebraic constraint unless the local model family is contained in that constraint set.

For the regular-family clause, if every $F_D^\eta=I+\eta X_D+O(\eta^2)$ on a fixed state space, then composing along $w$ gives $F_w=I+\eta\sum_{r=1}^K X_{w_r}+O(\eta^2)$.  The first-order term is linear in the generators and each generator enters with unit coefficient regardless of its position --- flat first-order impulse weights.  Hence the first-order term depends only on the multiset, every equal-multiset first-order contrast form vanishes identically for regular families, and the leading order effect is the $O(\eta^2)$ commutator of the regular expansion.

\subsection{Order-variance floor: proof and optimizer-class cases}
\label{app:variance-floor-proof}
This subsection proves Theorem~\ref{thm:variance-floor} and computes its kernel contrast
factor $V_K=\sum_{r=1}^{K}(W_{r,K}^{(s,\beta)}-\bar{W})^2$ for the three optimizer classes,
together with the fully closed-form two-domain case.  We present it after the flat-weights
derivation because the regular-optimizer case rests on that cancellation.

\paragraph{Proof of Theorem~\ref{thm:variance-floor}.}
By Propositions~\ref{prop:first-order-kernel} and~\ref{prop:second-order-transport}, the
scalar readout of a word $w$ expands as
\begin{equation}
    \inner{g_E}{\theta_K(w)-\theta_0}=\eta L(w)+\eta^2Q(w)+O(\eta^3),
    \qquad
    L(w)=-\inner{g_E}{A_KP_0m_0}-\sum_{r=1}^{K}a_r\phi_{w_r},
\end{equation}
with $a_r=W_{r,K}^{(s,\beta)}$, $\phi_i=\inner{g_E}{P_0g_i}$, and $Q$ the second-order
word-coefficient readout.  The warmed-buffer term is word-independent
(Proposition~\ref{prop:first-order-kernel}), so it cancels from every centered orbit moment
and $L$ may be replaced by the linear permutation statistic $-\sum_ra_r\phi_{w_r}$.

For a uniform random permutation $w$ of the fixed multiset, each position is marginally
uniform over the $K$ values, so $\operatorname{Var}(\phi_{w_r})=\sigma_\phi^2$ with
$\sigma_\phi^2=\frac{1}{K}\sum_i(\phi_i-\bar\phi)^2$; and for $r\ne t$, sampling without
replacement gives the finite-population identity
$\operatorname{Cov}(\phi_{w_r},\phi_{w_t})=-\sigma_\phi^2/(K-1)$.  Hence
\begin{equation}
    \operatorname{Var}_w\Bigl[\sum_{r=1}^{K}a_r\phi_{w_r}\Bigr]
    =\sigma_\phi^2\Bigl[\sum_ra_r^2-\frac{1}{K-1}\Bigl\{\Bigl(\sum_ra_r\Bigr)^2-\sum_ra_r^2\Bigr\}\Bigr]
    =\sigma_\phi^2\,\frac{K}{K-1}\sum_{r=1}^{K}(a_r-\bar a)^2,
\end{equation}
the classical Wald--Wolfowitz/Hoeffding linear permutation-statistic variance
\citep{waldwolfowitz1944,hoeffding1951}.  The
variance of the linear first-order statistic $L$ is therefore exact, not merely leading-order; the endpoint variance then carries the higher-order expansion, and
\begin{equation}
    \operatorname{Var}_w\bigl[\eta L+\eta^2Q+O(\eta^3)\bigr]
    =\eta^2\operatorname{Var}_w(L)+2\eta^3\operatorname{Cov}_w(L,Q)+O(\eta^4)
\end{equation}
gives Eq.~\eqref{eq:variance-floor} with its named remainder.

\paragraph{Regular optimizers.}
By the regular-family expansion of Appendix~\ref{app:fixed-clock-nondegen}, the first-order
impulse weights are flat: every generator enters the first-order term with unit coefficient
at every position, so $a_r\equiv\bar a$ and $V_K=0$ identically.  The first-order orbit
variance vanishes for every multiset and readout, and the leading order-induced variance is
carried by the $O(\eta^2)$ bracket term: $\operatorname{Var}_w=O(\eta^4)$.

\paragraph{Matched-clock control.}
Under the matched clock ($\beta=e^{-a\eta}$, $s_0=T_c/\eta$) the de-biasing denominator is
saturated: $\mathcal{D}=1-\beta^{s_0+t}=1-e^{-aT_c}\bigl(1+O(a\eta t)\bigr)$ is
$t$-independent to leading order.  Then
\begin{equation}
    W_{r,K}=\frac{1-\beta^{K-r+1}}{\mathcal{D}}
    =\frac{a\eta(K-r+1)}{\mathcal{D}}\Bigl(1-\frac{a\eta(K-r+1)}{2}+\cdots\Bigr),
\end{equation}
affine in position to leading order, so
\begin{equation}
    V_K=\frac{(a\eta)^2(K^3-K)}{12\,\mathcal{D}^2}\bigl(1+O(a\eta K)\bigr)=\Theta(\eta^2),
\end{equation}
placing the matched-clock first-order orbit variance at $\Theta(\eta^4)$, in the regular
class.  At the locked control parameters ($a=52.36$, $T_c=5.0$; Appendix~\ref{app:eta-slope-protocol}) the leading formula gives $V_K=2.330\times10^{-3}$ at
$\eta=5\times10^{-5}$ against the exact $V_K=2.2289\times10^{-3}$, and the $O(a\eta K)$
correction accounts for the slope depression observed at the large-$\eta$ end of the grid.

\paragraph{Fixed-clock AdamW.}
At the experimental configuration ($K=16$, $s=100$, $\beta_1=0.9$) the impulse weights fall
monotonically from $W_1=0.815$ to $W_{16}=0.100$ and $V_K=0.757=\Theta(1)$, so the
first-order orbit variance is $\Theta(\eta^2)$: the dichotomy of
Theorem~\ref{thm:variance-floor}.

\paragraph{Two-domain orbits in closed form.}
For balanced two-domain multisets $A^{K/2}B^{K/2}$ the linear statistic depends on $w$ only
through $S_A(w)=\sum_{r\in A\text{-positions}}a_r$, so the entire orbit distribution is the
distribution of $S_A$ over uniform $(K/2)$-subsets of positions.  In the two-point
idealization (every $A$ block at potential $\phi_A$, every $B$ block at $\phi_B$), the
monotone kernel and the rearrangement argument of
Appendix~\ref{app:potential-sorting} make the domain-sorted words the orbit extremes, with
range $\eta\,|\phi_A-\phi_B|\,S_{K,s,\beta}$ --- exactly the balanced-block contrast of
Eq.~\eqref{eq:balanced-c1} --- and dimensionless spread
\begin{equation}
    \mathrm{SD}/\mathrm{range}
    =\frac{\sqrt{K^2\sigma_W^2/\bigl(4(K-1)\bigr)}}{S_{K,s,\beta}},
    \qquad \sigma_W^2=V_K/K,
\end{equation}
fully determined by $(K,s,\beta_1)$: $0.1539$ at the experimental configuration.  With
within-domain batch-level spread of the $\phi_i$, the domain-sorted words are no longer the
exact orbit extremes and the ratio drifts upward (simulation: $+12\%$ at spread
$0.25\times$ the domain gap), which is the direction observed at orbit level in
Appendix~\ref{app:orbit-details}.
\subsection{\texorpdfstring{Frozen-$P$}{Frozen-P} first-order kernel}
Unrolling
\begin{equation}
    m_t=\beta m_{t-1}+(1-\beta)g_{w_t}
\end{equation}
with frozen gradients gives
\begin{equation}
    m_t=\beta^t m_0+(1-\beta)\sum_{r=1}^{t}\beta^{t-r}g_{w_r}.
\end{equation}
Substituting into
\begin{equation}
    \theta_K-\theta_0=-\eta\sum_{t=1}^{K}\frac{1}{1-\beta^{s+t}}P_0m_t
\end{equation}
yields Eq.~\eqref{eq:c1-word}.  For balanced blocks, the $m_0$ term cancels and direct summation gives Eq.~\eqref{eq:balanced-c1}.

\subsection{\texorpdfstring{Frozen-$P$}{Frozen-P} second-order transport}
\label{app:transport-weight-decay}
Write $g_D(\theta)=g_D+H_D(\theta-\theta_0)+O(\norm{\theta-\theta_0}^2)$ and expand $\theta_t=\theta_0+\eta x_t+\eta^2 y_t+O(\eta^3)$.  The first-order prefix before step $t$ is
\begin{equation}
    x_{t-1}=-\sum_{r<t}W_{r,t-1}P_0g_{w_r}-A_{t-1}P_0m_0
\end{equation}
when $\lambda=0$.  With decoupled weight decay, two $O(\eta^2)$ pieces appear.  First, the decay map scales the first-order displacement by $-\lambda$; this is included in the full frozen-gradient replay $\tau_{\mathrm{Adam}}$ because that replay evolves $\theta$ under decay.  Second, the prefix at which $H_{w_t}$ acts shifts by $-\lambda(t-1)\theta_0$; this is a live-gradient transport effect and belongs in the curvature residual $\Delta^{(1)}$.  Adding that prefix gives the weight-decay form of the coefficient.  In the fixed-target curl experiment, the reported $\lambda=0.01$ prefix contribution is conservative and telescopes out of the triangle-cycle measure.  Third derivatives contribute $O(\eta^3)$ to parameter displacement for fixed $K$.

\subsection{\texorpdfstring{Frozen-$P$}{Frozen-P} potential and sorting}
\label{app:potential-sorting}
For domains $i,j$, the balanced-block first-order edge has the form $e_{ij}^{(1)}=C(\phi_j-\phi_i)$ for $C>0$ and $\phi_i=\inner{g_E}{P_0g_i}$, so directed triangle sums telescope to zero.  For a fixed multiset, the first-order objective is $\sum_r W_{r,K}\phi_{w_r}$ with monotone weights $W_{r,K}$.  The rearrangement inequality gives both extrema: pairing largest potentials with largest weights maximizes this sum, and pairing them with smallest weights minimizes it.

\subsection{Second-order edge decomposition and triangle cycles}
\label{app:edge-decomp}
Let $K=2h$ and write $w_{ij}=i^h j^h$.  For any word $w$, Proposition~\ref{prop:second-order-transport} gives the scalar coefficient
\begin{equation}
    Y^{(1)}(w)=\sum_{D,q} C_{D,q}(w)M_{D,q}+\sum_D C_{D,0}(w)M_{D,0},
\end{equation}
where
\begin{equation}
    C_{D,q}(w)=\sum_{t:w_t=D}\sum_{r<t:w_r=q} W_{t,K}W_{r,t-1},
    \qquad
    C_{D,0}(w)=\sum_{t:w_t=D}W_{t,K}A_{t-1}.
\end{equation}
For $w_{ij}$, define
\begin{align}
    C_{ff}&=\sum_{t\le h}W_{t,K}\sum_{r<t}W_{r,t-1},\qquad
    C_{ss}=\sum_{t>h}W_{t,K}\sum_{h<r<t}W_{r,t-1},\\
    C_{\times}&=\sum_{t>h}W_{t,K}\sum_{r\le h}W_{r,t-1},\\
    C_{m,f}&=\sum_{t\le h}W_{t,K}A_{t-1},
    \qquad
    C_{m,s}=\sum_{t>h}W_{t,K}A_{t-1}.
\end{align}
Then
\begin{align}
Y^{(1)}(w_{ij})&=C_{ff}M_{i,i}+C_{\times}M_{j,i}+C_{ss}M_{j,j}
    +C_{m,f}M_{i,0}+C_{m,s}M_{j,0},\\
Y^{(1)}(w_{ji})&=C_{ff}M_{j,j}+C_{\times}M_{i,j}+C_{ss}M_{i,i}
    +C_{m,f}M_{j,0}+C_{m,s}M_{i,0}.
\end{align}
Subtracting gives Eq.~\eqref{eq:edge-decomp-main} with
\begin{equation}
    a_K=C_{ff}-C_{ss},\qquad b_K=C_{\times},\qquad c_K=C_{m,f}-C_{m,s}.
\end{equation}
In a triangle $i,j,k$, the self-curvature terms telescope:
\begin{equation}
    (M_{i,i}-M_{j,j})+(M_{j,j}-M_{k,k})+(M_{k,k}-M_{i,i})=0,
\end{equation}
and the warm-buffer terms telescope similarly.  Therefore the directed triangle sum (the discrete curl/circulation) is exactly Eq.~\eqref{eq:curl-main}.  If $P_0=I$, the antisymmetric term reduces to $b_K$ times the standard projected bracket $\inner{g_E}{H_jg_i-H_ig_j}$. The finite-clock coefficient $C_{\times}$ is the reason the raw contraction asymmetry should not be interpreted without clock weights; in the measured $\beta_1=0.9,s=100$ range, the cross coefficient is several times larger than the self-potential coefficient, motivating fixed-target triple tests with a common $g_E$.

\subsection{\texorpdfstring{Full-$v$}{Full-v} AdamW tangent coefficient}
\label{app:fullv-tangent}
This subsection gives the full AdamW second-order residual coefficient relative to full frozen-gradient AdamW replay.  Let
\begin{equation}
    b_t=1-\beta_1^{s+t},\qquad c_t=1-\beta_2^{s+t},
\end{equation}
and
\begin{equation}
    R_t(v)=\operatorname{Diag}\left(\frac{1}{b_t(\sqrt{v/c_t}+\epsilon)}\right).
\end{equation}
At $\eta=0$, frozen gradients evolve
\begin{align}
    m_t^0&=\beta_1m_{t-1}^0+(1-\beta_1)g_{D_t},\\
    v_t^0&=\beta_2v_{t-1}^0+(1-\beta_2)g_{D_t}^{\odot2}.
\end{align}
The first parameter coefficient of frozen replay is
\begin{equation}
    x_t=x_{t-1}-R_t(v_t^0)m_t^0-\lambda\theta_0,
    \qquad x_0=0.
\end{equation}
The live first perturbations obey
\begin{align}
    \dot g_t&=H_{D_t}x_{t-1},\\
    \dot m_t&=\beta_1\dot m_{t-1}+(1-\beta_1)\dot g_t,\\
    \dot v_t&=\beta_2\dot v_{t-1}+2(1-\beta_2)g_{D_t}\odot\dot g_t.
\end{align}
Coordinatewise,
\begin{equation}
    r'_{t,i}(v_i)=-\frac{1}{2b_t c_t\sqrt{v_i/c_t}(\sqrt{v_i/c_t}+\epsilon)^2}.
\end{equation}
Then the second-order live residual beyond frozen-gradient replay satisfies
\begin{equation}
    z_t=z_{t-1}-R_t(v_t^0)\dot m_t-\bigl(r'_t(v_t^0)\odot\dot v_t\odot m_t^0\bigr),
    \qquad z_0=0.
\end{equation}
Thus for readout $g_E$,
\begin{equation}
    \inner{g_E}{\theta_K^{\mathrm{live}}(w)-\theta_K^{\mathrm{FG}}(w)}=\eta^2\inner{g_E}{z_K(w)}+O(\eta^3).
\end{equation}
This derivation assumes the warmed second-moment coordinates used by the preconditioner are bounded away from zero; otherwise the square-root map needs a directional nonsmooth expansion.

\subsection{Backward-sensitivity (costate) recursion}
\label{app:costate}
The backward-sensitivity recursion is the standard costate form in elementary terms: a backward pass accumulates the sensitivity of the scalar endpoint readout to each intermediate optimizer state.  The scalar coefficient can be computed with an $O(K)$ HVP recursion.  Store $m_t^0,v_t^0,x_{t-1},R_t(v_t^0)$ and $r'_t(v_t^0)$ from the frozen replay.  Run backward with $\mu_{K+1}=\nu_{K+1}=0$:
\begin{align}
    \mu_u&=(1-\beta_1)R_u(v_u^0)g_E+\beta_1\mu_{u+1},\\
    \omega_u&=r'_u(v_u^0)\odot m_u^0\odot g_E,\\
    \nu_u&=\omega_u+\beta_2\nu_{u+1},\\
    \alpha_u&=\mu_u+2(1-\beta_2)g_{D_u}\odot \nu_u.
\end{align}
Then
\begin{equation}
    \inner{g_E}{z_K(w)}=-\sum_{u=1}^{K}\inner{\alpha_u}{H_{D_u}x_{u-1}}.
\end{equation}
Replacing $R_t(v_t^0)$ by $P_0/(1-\beta_1^{s+t})$ and setting $r'_t=0$ reduces this recursion to the frozen-$P$ coefficient in Proposition~\ref{prop:second-order-transport}.

\subsection{Trust-region coefficient scaling}
The trust-region ratio compares a second-order coefficient to a first-order coefficient.  The second-order coefficient contains products of prefix and endpoint weights,
\begin{equation}
    C^{(2)}_{D,j}(K)=\sum_{t=1}^{K}\sum_{r<t} W_{t,K}^{(s,\beta)}W_{r,t-1}^{(s,\beta)}\mathbf{1}\{w_t=D,w_r=j\},
\end{equation}
plus the analogous warm-momentum coefficient.  The experiments use these exact finite-$K$ coefficients.  The asymptotic $K^2$ summary describes their growth relative to the first-order coefficient; it is not a replacement for the exact computation.  Deviations from exponent two are expected when the first-order denominator approaches zero, when the finite-memory coefficient is not saturated, or when local curvature prefactors vary strongly across domain pairs.

\subsection{Finite-horizon accounting: diffusion versus drift}
\label{app:finite-horizon}
This subsection makes the Discussion's finite-horizon accounting precise.  The result is an accounting identity over windows in which the per-window local expansion of Theorem~\ref{thm:fixed-clock-singularity} holds.  The identity tracks how the per-window first-order channel composes to the endpoint and how that composition scales under different order policies, rather than giving endpoint control over a full run.

\begin{proposition}[Finite-horizon accounting]
\label{prop:finite-horizon}
Split a run into $N_{\mathrm{w}}$ windows of length $K$, and suppose that in each window $p$ the frozen-state equal-multiset order contrast is $\eta\,\zeta_p$ with the per-window remainder of Theorem~\ref{thm:fixed-clock-singularity}, that the suffix map transporting window $p$ to the endpoint readout has bounded sensitivity $G_p$, and that $\|G_p\|,\|\zeta_p\|\le B$ uniformly with per-window remainders $O(\eta^2)$.  Assume further that the suffix response is \emph{summable}: window $q$'s coefficient depends on an earlier window $p$'s displacement with a sensitivity that decays in $q-p$, so that $\sum_{q>p}\lVert \mathrm{D}\zeta_q\cdot G_{p\to q}\rVert\le C$ uniformly in $p$ (contractive or geometrically mixing suffix maps; mere non-expansiveness is \emph{not} enough, as noted below).  Then the endpoint order contrast satisfies
\begin{equation}
\Delta_{\mathrm{end}}=\eta\sum_{p=1}^{N_{\mathrm{w}}}G_p\zeta_p+\mathcal{E},\qquad \|\mathcal{E}\|\le CN_{\mathrm{w}}\eta^2,
\end{equation}
with $C$ depending only on $B$ and the uniform Lipschitz constants.  Consequently:
\begin{enumerate}
\item \emph{(Diffusion.)} If $\{G_p\zeta_p\}$ has mean zero with uniformly summable autocovariances ($\sup_p\sum_q\lvert\operatorname{Cov}(G_p\zeta_p,G_q\zeta_q)\rvert\le C'$)---the mean-zero weak-dependence model for independently reshuffled windows, of which uncorrelated (martingale-difference) increments are the special case---then $\mathbb{E}\,\Delta_{\mathrm{end}}=O(N_{\mathrm{w}}\eta^2)$ and $\operatorname{Var}(\Delta_{\mathrm{end}})=O(\eta^2N_{\mathrm{w}})$, so the centered root-mean-square endpoint contrast is $O(\eta\sqrt{N_{\mathrm{w}}})$.  At $N_{\mathrm{w}}\sim1/\eta$ this is $O(\sqrt{\eta})\to0$.
\item \emph{(Drift.)} If the transported coefficients have nonzero Ces\`aro mean, $\frac1{N_{\mathrm{w}}}\sum_{p=1}^{N_{\mathrm{w}}}\mathbb{E}[G_p\zeta_p]\to\mu\neq0$---the model for repeated packed order, a sorted curriculum, no reshuffling, or a fixed domain-order skew---then $\mathbb{E}\,\Delta_{\mathrm{end}}=\eta N_{\mathrm{w}}\mu+o(\eta N_{\mathrm{w}})$, which is formally $O(1)$ at the regime boundary $N_{\mathrm{w}}\sim1/\eta$.
\item \emph{(Matched clock.)} Under the matched-clock control the per-window leading coefficient is $O(\eta^2)$ rather than $O(\eta)$, so each bound gains one power of $\eta$: diffusion becomes $O(\eta^{3/2})$ and drift $O(\eta)$ at $N_{\mathrm{w}}\sim1/\eta$.
\end{enumerate}
\end{proposition}

\begin{proof}
The identity is a telescoping sum of the per-window first-order displacements composed with their bounded suffix transports.  The remainder collects $N_{\mathrm{w}}$ per-window $O(\eta^2)$ terms together with the cross-window products of first-order displacements, each $O(\eta^2)$ under the uniform bounds; under the summable-response assumption these corrections sum to the linear-in-$N_{\mathrm{w}}$ bound $\|\mathcal{E}\|\le CN_{\mathrm{w}}\eta^2$.  Non-expansiveness alone does not suffice: window $q$ inherits a prefix-displacement shift of order $\eta q$---the sum of $q$ earlier $O(\eta)$ displacements, transported without magnification but not cancelled---so its coefficient acquires an $O(\eta q)$ correction and contributes $O(\eta^2 q)$, summing to $\|\mathcal{E}\|\le CN_{\mathrm{w}}^2\eta^2$; the linear bound therefore needs contraction or summable memory, not non-expansiveness.  For (1), mean-zero increments give $\mathbb{E}\sum_pG_p\zeta_p=0$, and uniformly summable autocovariances give $\operatorname{Var}\bigl(\sum_pG_p\zeta_p\bigr)=\sum_{p,q}\operatorname{Cov}(G_p\zeta_p,G_q\zeta_q)\le N_{\mathrm{w}}\sup_p\sum_q\lvert\operatorname{Cov}(G_p\zeta_p,G_q\zeta_q)\rvert=O(N_{\mathrm{w}})$ (uncorrelated martingale-difference increments are the special case, where the double sum collapses to $\sum_p\operatorname{Var}(G_p\zeta_p)$), so $\operatorname{Var}(\eta\sum_pG_p\zeta_p)=O(\eta^2N_{\mathrm{w}})$ (the random part of $\mathcal{E}$ contributes variance $O(N_{\mathrm{w}}^2\eta^4)$, lower order at $N_{\mathrm{w}}\sim1/\eta$) and the mean of $\Delta_{\mathrm{end}}$ is carried by $\mathcal{E}=O(N_{\mathrm{w}}\eta^2)$.  For (2), $\mathbb{E}\,\eta\sum_pG_p\zeta_p=\eta N_{\mathrm{w}}\mu+o(\eta N_{\mathrm{w}})$ by the assumed Ces\`aro convergence, and this dominates the $O(N_{\mathrm{w}}\eta^2)$ remainder whenever $\eta\to0$ with $N_{\mathrm{w}}\eta$ held fixed.  For (3), substituting the matched-clock per-window coefficient, which is in the regular $O(\eta^2)$ class rather than the singular $O(\eta)$ class, multiplies every term by an additional factor of $\eta$.
\end{proof}

On a real model the coherent case exhibits clean linear-in-$N_{\mathrm{w}}$ drift only while the local expansion holds, after which the warm state moves enough to require recomputation.  An empirical study would therefore test the accumulation \emph{mode}---sublinear spread under reshuffling versus linear drift under coherent order---rather than literal $O(1)$ separation at full training budget.  The summable-response hypothesis is itself an idealization---it holds under contraction or geometric mixing of the suffix maps, but not under mere non-expansiveness (where the $O(\eta q)$ prefix accumulation already pushes the remainder to $O(N_{\mathrm{w}}^2\eta^2)$), and training-map Jacobians can in any case exceed unit norm.  The diffusion-versus-drift dichotomy is therefore a linearized first-variation accounting valid in the regime $N_{\mathrm{w}}\eta\ll1$ where the per-window local expansion holds.

\section{Experimental Protocol Details}
\label{app:protocol}

\paragraph{Rows and aggregation.}
A measurement row is a tuple consisting of a model checkpoint, domain pair, seed, block length $K$, and the corresponding warmed optimizer state.  The homogeneous second-order local-validity analysis contains 192 model-pair-block rows from 48 model-block JSONL files and 10 checkpoints.  The 192 rows are the measurement rows in the released artifact, not a filtered subset of a larger second-order table.  The breakdown is 160 Qwen/Pythia rows, 16 Llama-3.2-3B rows, and 16 Gemma-2-2B rows.  The Qwen/Pythia checkpoints are Qwen-3-0.6B, Qwen-3-1.7B, Qwen-3-4B, Qwen-3-14B, Pythia-70M, Pythia-160M, Pythia-410M, and Pythia-1B.  Main-text fits are reported both at the row level and after aggregating pairs into model/K medians.  Model/K medians reduce the influence of denominator singularities in individual pairs, while row-level fixed-effect fits preserve pair-level variation.

\paragraph{Coefficient/outcome separation.}
Coefficient rows are generated separately from realized-outcome analysis.  Each coefficient row contains the six curvature contractions $M_{A,A},M_{A,B},M_{B,A},M_{B,B},M_{A,0},M_{B,0}$, the coefficient-weighted pieces $\Delta^{(1)}_{\mathrm{grad}}$ and $\Delta^{(1)}_{m_0}$, and their sum $\Delta^{(1)}$.  Coefficient rows are joined to realized AB/BA outcomes only for analysis.  Generating coefficient rows separately avoids tuning the correction after inspecting live residuals.

\paragraph{Fixed-target triple protocol.}
The transitivity-radius experiment (Section~\ref{sec:sorting-curl}) uses Llama-3.2-1B with fp32 base weights and fp32 LoRA adapters; the driver aborts on any non-fp32 configuration because bf16 base noise can dominate antisymmetric cycle-driving differences.  Seven Pile domains (math, code, news, stackexchange, wikipedia, legal, biomedical) are loaded with disjoint 512/512 train/eval slices per domain.  Each domain serves in turn as the fixed target: its evaluation-slice covector $g_E$ is computed once and shared by every pairwise edge of every triple among the remaining six source domains, which is the condition under which triangle-cycle structure is well defined.  After $100$ AdamW warmup steps cycling all domains, the warmed preconditioner $P_0$ and first moment $m_0$ are frozen, and one HVP per source domain $h_D=H_D(P_0g_E)$ yields all contractions $M_{D,j}=\inner{h_D}{P_0g_j}$ by Hessian symmetry, so the curvature cost is linear rather than quadratic in the number of sources.  Edges are evaluated at $K\in\{16,24,32,48,64\}$ through the finite-clock coefficients alone, with gradients and HVPs averaged over $16$ batches of $4$ sequences at length $512$.  Each (triple,$K$) cell records the directed-3-cycle indicator for the computed frozen-$P$ coefficient edge field and, separately, for the potential-only and bracket-only components, plus the orientation-aware bracket cycle term.  The protocol reserves live AB/BA spot comparisons from a bit-identical fork for selected high-$\rho_{\mathrm{curl}}$ triples and matched low-$\rho_{\mathrm{curl}}$ controls; those comparisons test regime behavior rather than validating each cyclic cell.  The protocol, including scope boundaries, is included in the artifact as \texttt{configs/fixed\_target\_triples\_protocol.json}.

\paragraph{Ordering-orbit and variance-scaling protocol.}
The order-variance experiments behind Theorem~\ref{thm:variance-floor}
(Section~\ref{sec:variance-floor}) run in three legs.  The orbit experiment measures eight
LoRA cells (Pythia-1B and Llama-3.2-1B, two domain pairs, fixed-clock AdamW and matched-clock
arms) with $64$ orderings per cell under bitwise-deterministic replay; per-cell determinism
and noise-floor criteria are enforced, and all rows satisfy those criteria.  The
$\eta$-scaling rerun repeats the two lowest-slope cells at $\eta=1.25\times10^{-5}$ against
the predicted slope and deficit values.  The variance-scaling experiment sweeps five
learning rates spanning $16\times$ in three optimizer arms (fixed-clock AdamW, matched-clock
momentum, SGD) across four model--pair series and reports Theil--Sen slopes of
$\log\operatorname{Var}$ versus $\log\eta$ with bootstrap $95\%$ confidence intervals over
series; per-cell $\phi$ records enable the $\sigma_\phi^2(\eta)$ decomposition reported in
Section~\ref{sec:variance-floor}.  The exclusion criteria and tolerances are listed with the protocol details in the artifact.

\paragraph{HVP implementation.}
For token-mean causal language-model losses, microchunking computes HVPs of unnormalized loss sums and divides by total valid tokens in the original outer batch.  The normalization gives the HVP of the token-mean objective rather than an equal-weight average over microchunks.  Representative chunked-versus-unchunked comparisons found few-percent drift in $\Delta^{(1)}$, so the analysis does not rely on sub-percent calibration.

\paragraph{Trust-region metrics.}
For each cell the analysis reports
\begin{equation}
R_{\tau,0}=\frac{\Delta_{\mathrm{live}}}{\tau_{\mathrm{Adam}}},\qquad
R_{\tau,0+1}=\frac{\Delta_{\mathrm{live}}}{\tau_{\mathrm{Adam}}+\Delta^{(1)}},
\end{equation}
additive residual reduction, and $\rho=|\Delta^{(1)}/\tau_{\mathrm{Adam}}|$.  The main trust-region figures use $\rho$ rather than the corrected ratio alone.  The corrected ratio can be misleading when $\tau_{\mathrm{Adam}}+\Delta^{(1)}$ crosses near zero; additive residuals ($\Delta_{\mathrm{live}}-\tau_{\mathrm{Adam}}$) are analyzed in Appendix~\ref{app:additional}.

\paragraph{Small denominators.}
Rows with very small $|\tau_{\mathrm{Adam}}|$ can have extremely large or unstable raw $\rho$.  Small-denominator rows are not discarded: a zero of the first-order replay coefficient is itself a genuine singularity of the local first-order calculus.  Such rows can dominate global model-size fits, so we report model/K medians as robustness summaries and treat local curvature prefactors and denominator zeros as the relevant scale variables.

\paragraph{Curvature prefactors.}
For fixed model/pair/seed, the measured curvature contractions are constant across $K$; the block-length dependence comes from the finite-clock coefficients.  The prefactor analysis therefore asks whether $\rho/K^2$ is explained by local curvature-block quantities such as $|M_{A,B}-M_{B,A}|$ and $|M_{A,A}-M_{B,B}|$.  The resulting comparison links the analytic coefficient scaling to the measured gradient/Hessian geometry.

\paragraph{Local order diagnostic.}
Algorithm~\ref{alg:diagnostic-app} summarizes the local order diagnostic implied by the expansion.  The diagnostic decides whether first-order reasoning about order is locally well conditioned before treating a schedule as a design variable.
\begin{algorithm}[h]
\caption{Local AdamW order diagnostic}
\label{alg:diagnostic-app}
\begin{algorithmic}[1]
\Require checkpoint and AdamW state, candidate domains $\{D_i\}$, evaluation covector $g_E$, block length $K$
\State Estimate local gradients $g_i$ and frozen preconditioner $P_0$
\State Compute first-order potentials $\phi_i=\inner{g_E}{P_0g_i}$
\State Form the finite-clock weights $W_{r,K}^{(s,\beta)}$ and the memory response $\tau_{\mathrm{Adam}}$
\State Estimate curvature contractions $M_{D,j}=\inner{g_E}{P_0H_DP_0g_j}$ and warmed-momentum terms when needed
\State Compute $\Delta^{(1)}$
\If{$|\tau_{\mathrm{Adam}}|$ is below the denominator tolerance}
    \State mark a first-order denominator singularity and do not use $\rho$ alone as a validity estimate
\Else
    \State Compute $\rho=|\Delta^{(1)}/\tau_{\mathrm{Adam}}|$
\If{$\rho\ll1$}
    \State use the first-order potential/sorting rule as a local approximation
\ElsIf{$\rho\approx1$}
    \State treat second-order transport as comparable to first-order replay and inspect the curvature-corrected local approximation
\Else
    \State treat the proposed order as outside the local trust region
\EndIf
\EndIf
\end{algorithmic}
\end{algorithm}

\paragraph{Held-out-NLL calibration.}
The held-out-NLL calibration in Section~\ref{sec:variance-floor} uses a near-canceling mix comparison: Llama-3.2-1B LoRA SFT, fp32 base, math/code data, 512 training steps, 25 order seeds per side, and a fixed mixed held-out NLL probe.  Configuration A is a 50/50 math--code mix and configuration B is a 70/30 mix.  The $\eta$ sweep keeps the mix definitions, example-pool construction, seed list, and readout fixed, and varies only AdamW's learning rate over $\{10^{-4},2\times10^{-4}\}$, with the banked $5\times10^{-5}$ run as the baseline.  The degeneracy criteria treat an $\eta$ point as uninterpretable if mean held-out NLL rises more than 0.15 above the baseline, mean last-epoch train loss rises more than 0.5, more than 20\% of rows are excluded, or $\sigma_{\mathrm{ord}}/\text{mean}>0.02$.  Both eta-sweep points satisfy these criteria without exclusions.  The reported sign-change rate is the fraction of the $25\times25$ cross-seed A/B pairs whose held-out-NLL sign disagrees with the mean-gap sign at that eta.  These pairs reuse the same 25 per-side seeds, so the rate is a descriptive grid readout over dependent comparisons rather than an independent-sample statistic.

\section{Additional Results}
\label{app:additional}

\subsection{Residual-validation bins}
Table~\ref{tab:resid-bins-app} gives the stratified residual-validation statistics behind Figure~\ref{fig:residual-validation}.  The target is the live residual $\Delta_{\mathrm{live}}-\tau_{\mathrm{Adam}}$; the predictor is $\Delta^{(1)}$.

\begin{figure}[t]
\centering
\includegraphics[width=0.78\linewidth]{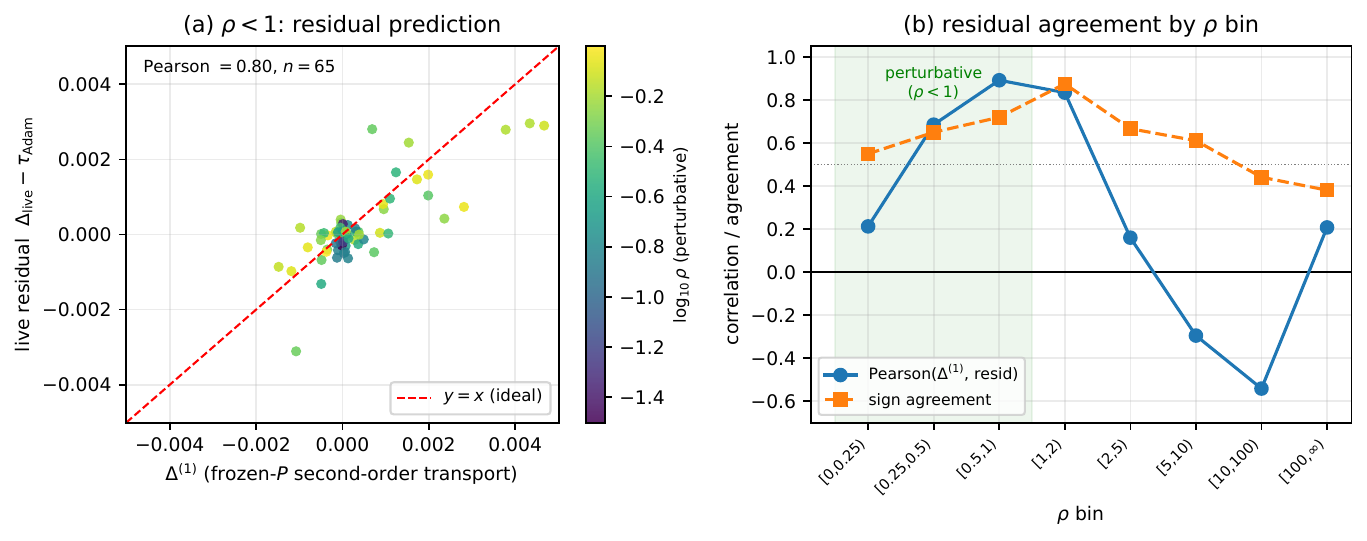}\\[1.0ex]
\includegraphics[width=0.78\linewidth]{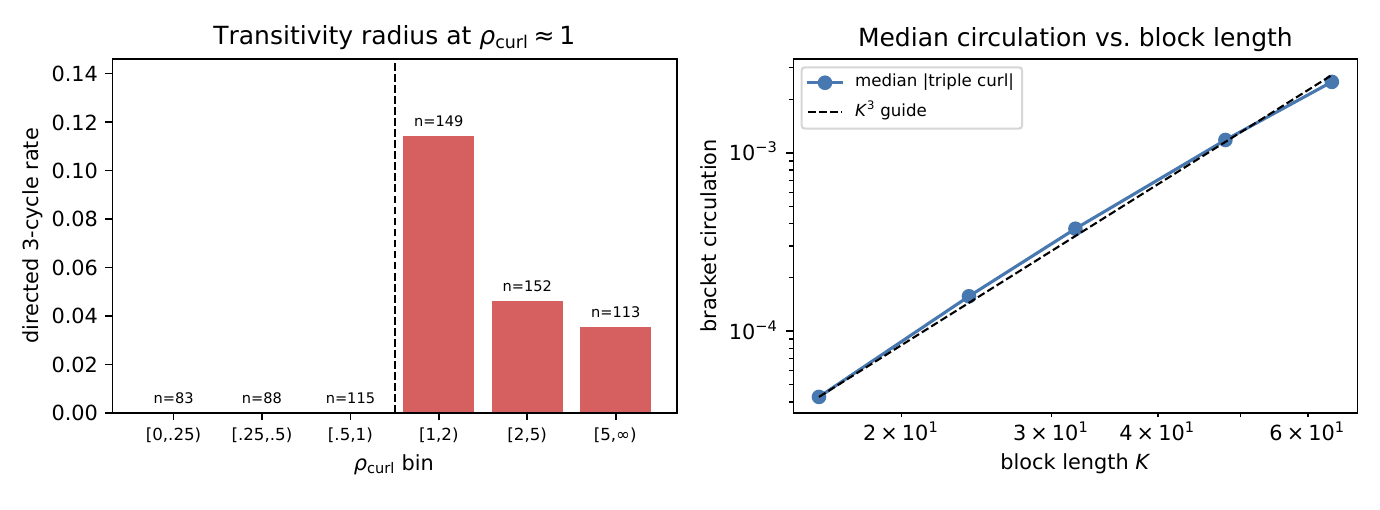}
\caption{Local validity and scalar sortability.  Top left: in the $\rho<1$ subset, $\Delta^{(1)}$ tracks the live residual $\Delta_{\mathrm{live}}-\tau_{\mathrm{Adam}}$; top right: binning by $\rho$ shows strongest agreement in the intermediate low-$\rho$ bins, a noise-limited very-low-$\rho$ bin, and deterioration at larger ratios.  Bottom left: in the fixed-target coefficient experiment, the frozen-$P$ edge field has zero directed 3-cycles when every edge satisfies $\rho_{\mathrm{curl}}<1$, and a finite cycle rate above the radius.  Bottom right: median bracket circulation grows with block length with an empirical $K^3$ guide.}
\label{fig:residual-validation}
\end{figure}

\begin{table}[h]
\centering
\small
\caption{Residual validation stratified by $\rho$.}
\label{tab:resid-bins-app}
\begin{tabular}{lrrrrr}
\toprule
$\rho$ bin & $n$ & slope & Pearson & Spearman & sign \\
\midrule
$[0,0.25)$ & 20 & 0.40 & 0.21 & 0.17 & 55\% \\
$[0.25,0.5)$ & 20 & 1.07 & 0.69 & 0.66 & 65\% \\
$[0.5,1)$ & 25 & 0.60 & 0.89 & 0.86 & 72\% \\
$[1,2)$ & 24 & 0.78 & 0.83 & 0.90 & 88\% \\
$[2,5)$ & 30 & 0.06 & 0.16 & 0.45 & 67\% \\
$[5,10)$ & 18 & -0.05 & -0.30 & 0.23 & 61\% \\
$[10,100)$ & 34 & -0.02 & -0.54 & -0.18 & 44\% \\
$[100,\infty)$ & 21 & 0.00 & 0.21 & -0.04 & 38\% \\
\midrule
$\rho<1$ & 65 & 0.68 & 0.80 & 0.65 & 65\% \\
$\rho\ge5$ & 73 & 0.00 & -0.04 & -0.14 & 47\% \\
\bottomrule
\end{tabular}
\end{table}

\noindent Prediction quality is strongest for $\rho\lesssim2$ (slope toward one, Pearson up to $0.89$, sign agreement rising to $88\%$) and falls off or changes sign for $\rho\ge5$.

\subsection{Baseline table}
Table~\ref{tab:baseline-app} reports the baseline predictors behind the diagnostic summary in Section~\ref{sec:experiments}.  Spearman uses standard tied ranks, which matters for the discrete block-length predictor $K$.

\begin{table}[h]
\centering
\small
\caption{Predictors of first-order approximation error on 192 rows.}
\label{tab:baseline-app}
\begin{tabular}{lrr}
\toprule
Predictor & Spearman with $|R_{\tau,0}-1|$ & AUC sign error \\
\midrule
$\rho$ & 0.629 & 0.790 \\
$g_A$ norm & 0.345 & 0.657 \\
$|\Delta^{(1)}|$ & 0.299 & 0.691 \\
$g_B$ norm & 0.297 & 0.701 \\
curvature norm & 0.273 & 0.677 \\
curvature cross-asymmetry & 0.186 & 0.650 \\
$K$ & 0.170 & 0.565 \\
$|\tau_{\mathrm{Adam}}|$ & 0.017 & 0.590 \\
\bottomrule
\end{tabular}
\end{table}

\subsection{Clustered uncertainty and denominator controls}
\label{app:stat-controls}
Table~\ref{tab:cluster-ci-app} reports clustered bootstrap intervals for the key diagnostics.  The rows are clustered either by model/pair/block file or by model/pair, rather than treated as independent scalar observations.  The conclusions are unchanged: $\rho$ remains a strong trust-region and sign-error predictor, and $\Delta^{(1)}$ remains predictive of the live residual in the low-$\rho$ regime.

\begin{table}[h]
\centering
\small
\caption{Clustered uncertainty and denominator controls on the 192-row diagnostic grid.  Intervals are percentile bootstrap 95\% CIs.}
\label{tab:cluster-ci-app}
\begin{tabular}{lccc}
\toprule
Statistic & point & file-cluster CI & model/pair CI \\
\midrule
$\rho$ Spearman with $|R_{\tau,0}-1|$ & 0.629 & [0.53, 0.71] & [0.50, 0.73] \\
$\rho$ AUC for sign error & 0.790 & [0.71, 0.86] & [0.70, 0.86] \\
$\Delta^{(1)}$ Pearson with residual, $\rho<1$ & 0.797 & [0.63, 0.90] & [0.65, 0.93] \\
OLS slope, residual on $\Delta^{(1)}$, $\rho<1$ & 0.684 & [0.56, 0.78] & [0.52, 0.81] \\
$\rho$ Spearman with $|\Delta_{\mathrm{live}}-\tau|$ & 0.619 & [0.47, 0.70] & [0.45, 0.76] \\
$|\Delta^{(1)}|$ Spearman with $|\Delta_{\mathrm{live}}-\tau|$ & 0.851 & [0.76, 0.89] & [0.76, 0.90] \\
partial Spearman: $\rho$ vs. $|R_{\tau,0}-1|\mid |\tau|$ & 0.671 & [0.57, 0.75] & [0.52, 0.78] \\
partial Spearman: $\rho$ vs. $|\Delta_{\mathrm{live}}-\tau|\mid |\tau|$ & 0.573 & [0.48, 0.65] & [0.44, 0.67] \\
\bottomrule
\end{tabular}
\end{table}

The normalized magnitude target $|R_{\tau,0}-1|$ shares the denominator $\tau_{\mathrm{Adam}}$ with $\rho$.  The last two rows therefore residualize tied ranks against $|\tau_{\mathrm{Adam}}|$.  The ratio remains predictive under this partial-rank control.  For raw residual magnitude, the numerator $|\Delta^{(1)}|$ is naturally strongest; we use $\rho$ for its intended role as a normalized trust-region and sign-error diagnostic.

\subsection{Finite-window exponent uncertainty}
\label{app:eta-cis}
Table~\ref{tab:eta-ci-app} gives cluster bootstrap intervals for the $\eta$-slope experiment, resampling domain-pair/seed clusters.  The fixed-clock optimizer arms are near-linear but not exactly unit-slope over the finite tested window: AdamW is slightly below one and Lion slightly above one.  The finite-window spread is compatible with the asymptotic fixed-clock theorem, which asserts a nonzero first-order coefficient rather than an exact finite-window exponent of one for every optimizer implementation.

\begin{table}[h]
\centering
\small
\caption{Finite-window $\eta$-slope fits with pair/seed clustered CIs.}
\label{tab:eta-ci-app}
\begin{tabular}{llcc}
\toprule
Model & optimizer & slope $p$ & 95\% CI \\
\midrule
Pythia-1B & SGD & 2.005 & [1.905, 2.014] \\
Pythia-1B & matched-clock & 1.973 & [1.969, 1.973] \\
Pythia-1B & AdamW & 0.910 & [0.870, 0.956] \\
Pythia-1B & Lion & 1.072 & [1.052, 1.083] \\
Llama-3.2-1B & SGD & 1.960 & [1.955, 2.067] \\
Llama-3.2-1B & matched-clock & 1.968 & [1.963, 1.970] \\
Llama-3.2-1B & AdamW & 0.915 & [0.867, 0.967] \\
Llama-3.2-1B & Lion & 1.094 & [1.079, 1.111] \\
\bottomrule
\end{tabular}
\end{table}

\subsection{Finite-clock coefficient scaling}
Table~\ref{tab:kernel-coeff-app} reports the finite-$K$ coefficient scales used to interpret the block-length law.  The experiments use exact finite-$K$ coefficients; the $K^2$ notation is a regime summary.

\begin{table}[h]
\centering
\small
\caption{Finite-clock first- and second-order coefficient scales at $\beta_1=0.9,s=100$.  Here $\max|C^{(2)}|$ is the largest gradient-transport coefficient ($C_{ff},C_{ss},C_\times$ of Appendix~\ref{app:edge-decomp}); the warm-buffer coefficients $C_{m,f},C_{m,s}$ (carrying the $m_0$ term) are excluded.}
\label{tab:kernel-coeff-app}
\begin{tabular}{rrrr}
\toprule
$K$ & $Kc_K$ & $\max |C^{(2)}|$ & $\max |C^{(2)}|/(Kc_K)$ \\
\midrule
16 & 29.2 & 11.5 & 0.39 \\
24 & 46.3 & 42.0 & 0.91 \\
32 & 59.7 & 100.7 & 1.69 \\
48 & 76.2 & 316.9 & 4.16 \\
64 & 83.9 & 671.5 & 8.00 \\
128 & 89.8 & 3431.6 & 38.22 \\
256 & 90.0 & 15142.0 & 168.25 \\
\bottomrule
\end{tabular}
\end{table}

\noindent The second-order/first-order ratio crosses one between $K=24$ (ratio $0.91$) and $K=32$ (ratio $1.69$) and grows steeply thereafter, so no single block length is universally safe.

\subsection{Eta-slope protocol and matched-clock control}
\label{app:eta-slope-protocol}
The $\eta$-slope experiment uses Pythia-1B and Llama-3.2-1B, fp32 base models, $K=16$, three domain pairs, three seeds, three deterministic replicates, and five learning rates.  The matched-clock linear-memory arm in Figure~\ref{fig:eta-slope} is not ordinary fixed-$\beta$ heavy-ball momentum.  The matched-clock arm is a deliberately regularized first-moment control:
\begin{equation}
    m_{t+1}=\beta(\eta)m_t+(1-\beta(\eta))g_{w_t}(\theta_t),
    \qquad \beta(\eta)=\exp(-a\eta),
\end{equation}
\begin{equation}
    \widehat m_{t+1}=\frac{m_{t+1}}{1-\beta(\eta)^{s_0+t+1}},
    \qquad s_0=T_c/\eta,
    \qquad \theta_{t+1}=\theta_t-\eta\widehat m_{t+1}.
\end{equation}
Thus the bias-correction clock is matched to the continuous-time training variable $\tau=\eta k$, and the EMA window scales like $1/\eta$ steps.  Here $a$ and $T_c$ are $\eta$-independent constants ($a=52.36$, $T_c=5.0$ in our runs); the per-cell realized decay $\beta(\eta)=e^{-a\eta}$ and bias clock $s_0=T_c/\eta$ are logged in every matched-clock row's metadata in the artifact.  The continuous-time averaging window $1/a$ is held fixed across the sweep, while its step-equivalent $1/(a\eta)$ \emph{grows} as $\eta\to0$; this growth is precisely what flattens the impulse-weight spread to $O(\eta)$ and returns the variance to regular $O(\eta^4)$ scaling.  Under this matched-clock scaling the first-order impulse weights flatten in the $\eta\to0$ limit for equal-multiset words, restoring the regular $O(\eta^2)$ commutator regime.  The matched-clock control shows that the observed $p\approx1$ behavior is not a generic artifact of taking larger parameter steps; the behavior is a property of fixed-clock state.  Plain fixed-$\beta$ momentum, unlike this matched-clock control, is predicted to remain in the first-order class.  Measured exponents and clustered CIs are in Table~\ref{tab:eta-ci-app}.  Lion is a nonsmooth empirical comparison for fixed-clock state; the formal closed-form transport coefficient is the AdamW/frozen-$P$ theory.

\subsection{Qwen-3-8B first-order sweep}
A separate Qwen-3-8B first-order sweep measured $c_K$ behavior but did not contain the second-order fields required to compute $\rho$.  The Qwen-3-8B sweep is therefore not counted in the homogeneous 192-row second-order diagnostic grid.  We use the 192-row grid for all $\rho$, residual-validation, baseline, and conservativity-gap summaries.
\subsection{Model size is not the local scale variable}
Pythia and Qwen show that parameter count alone is not a stable scale variable for this local diagnostic.  Small Pythia checkpoints are high-$\rho$ throughout much of the range, while the 1B checkpoint enters a more perturbative regime; Qwen-family medians are nonmonotone after including Qwen-14B.  We therefore treat local curvature prefactors and denominator zeros as the relevant scale variables.

\subsection{Edge-level motivation for fixed-target triples}
\label{app:curl-motivation}
Equation~\eqref{eq:edge-decomp-main} identifies a cycle-driving bracket component inside the frozen-$P$ transport coefficient.  The existing 192-row grid is not a valid directed-cycle dataset because each pair row uses its own pair-specific evaluation covector.  Nevertheless, those rows can measure whether the coefficient-weighted bracket term is large enough to justify a fixed-target triple experiment.  Define the second-order-only bracket share
\begin{equation}
q_{\mathrm{curl}}(i,j)=\frac{|b_K(M_{ji}-M_{ij})|}{|b_K(M_{ji}-M_{ij})|+|a_K(M_{ii}-M_{jj})|+|c_K(M_{i0}-M_{j0})|+\epsilon}.
\end{equation}
Across the 192 pairwise rows, the median $q_{\mathrm{curl}}$ is $0.25$, the 90th percentile is $0.69$, and the maximum is $0.90$.  Using the transitivity denominator in Eq.~\eqref{eq:rhocurl}, the fraction of edges with $\rho_{\mathrm{curl}}\ge1$ rises with block length: $0.18,0.18,0.30,0.35,0.50$ for $K=16,24,32,48,64$.  These pairwise numbers do not measure directed cycles; they only show that the clock-weighted bracket component is not negligible in the same measured cells.  A valid cycle-rate test must hold the target covector fixed across all edges of each domain triple.  The released GPU protocol therefore computes a common $g_E$, one HVP $H_D(P_0g_E)$ per source domain, all $M_{D,j}$ values by dot products, and then evaluates cycle rate as a function of $K$ and $\rho_{\mathrm{curl}}$.

\subsection{Additional order-orbit details}
\label{app:orbit-details}
The main text reports the ordering-orbit result as a mechanism measurement rather than as a large table.  The detailed cell-level values (Table~\ref{tab:orbit-details-app}) are retained here because they are useful for evaluating the frozen-$P$ scope.  At $\eta=5\times10^{-5}$, fixed-clock AdamW shows high ordering-by-ordering agreement between live readout and the closed-form first-order kernel, with slopes below one due to the measured $\eta^3$ cross-term in Theorem~\ref{thm:variance-floor}.

\paragraph{Scope of the variance floor.}\label{app:floor-scope} Writing the order variance as $\eta^2 \mathcal{A}_K+\eta^3 \mathcal{B}_K+O(\eta^4)$, with leading coefficient $\mathcal{A}_K=\frac{K}{K-1}\sigma_\phi^2 V_K$ from Eq.~\eqref{eq:variance-floor} and $\mathcal{B}_K$ the $\eta^3$ transport covariance, whenever the correction is perturbative---$|\eta^3 \mathcal{B}_K|\le\gamma\,\eta^2 \mathcal{A}_K$ with $\gamma<1$ for all $\eta\le\eta_0$---the order variance obeys $\operatorname{Var}_w\ge(1-\gamma)\,\eta^2 \mathcal{A}_K$, an $\eta^2$ lower bound with no counterpart for the regular $O(\eta^4)$ optimizers.  The floor is not a pointwise bound for arbitrary finite $\eta$: because $\mathcal{B}_K$ can be negative, individual finite-$\eta$ cells can fall below the asymptotic line, which is the source of the sub-one ordering slopes above.
\begin{table}[h]
\centering
\small
\caption{Ordering-orbit agreement by cell, 64 deterministic orderings per cell.}
\label{tab:orbit-details-app}
\begin{tabular}{lcc}
\toprule
Cell & slope & correlation \\
\midrule
Pythia-1B math--code & 0.801 & 0.976 \\
Pythia-1B code--dialogue & 0.622 & 0.857 \\
Llama-3.2-1B math--code & 0.759 & 0.891 \\
Llama-3.2-1B code--dialogue & 0.758 & 0.754 \\
\bottomrule
\end{tabular}
\end{table}

The matched-clock cells return slope and correlation equal to one within deterministic replay tolerance, confirming that the same ordering harness preserves the kernel when the memory clock is regularized.  At the locked learning rate, the fixed-clock orbit variance is much larger than the matched-clock control (Table~\ref{tab:matched-suppression-app}):
\begin{table}[h]
\centering
\small
\caption{Orbit-variance suppression by the matched-clock control.  This magnitude is dominated by the missing AdamW preconditioner---the $\sigma_\phi^2$ ratio---rather than the clock, so we do not read it as a clock effect; the clock signature is the variance-slope class separation, not this suppression magnitude.}
\label{tab:matched-suppression-app}
\begin{tabular}{lc}
\toprule
Cell & $\operatorname{Var}[\mathrm{AdamW}]/\operatorname{Var}[\mathrm{matched\ clock}]$ \\
\midrule
Pythia-1B math--code & $3.9\times10^8$ \\
Pythia-1B code--dialogue & $1.9\times10^8$ \\
Llama-3.2-1B math--code & $1.6\times10^7$ \\
Llama-3.2-1B code--dialogue & $6.4\times10^5$ \\
\bottomrule
\end{tabular}
\end{table}

Two lower-$\eta$ reruns were used as a remainder comparison.  Llama code--dialogue matched the frozen prediction closely: slope $0.935$ versus predicted $0.9395$, and deficit scaling $0.269$ versus predicted $0.25$.  Pythia code--dialogue preserved the kernel correlation ($0.987$) but retained an $\eta$-independent amplitude factor of about $0.74$, consistent with the full-$v$ renormalization channel rather than a breakdown of the first-order clock law.  In the variance-scaling sweep, the aggregate AdamW variance exponent is below the fixed-$\sigma_\phi$ ideal value of two because $\sigma_\phi^2$ itself drifts with learning rate; using the measured $d\log\sigma_\phi^2/d\log\eta\approx -0.69$ predicts a variance exponent near $1.31$, close to the fitted $1.20$.

\subsection{Direct fixed-state replay}
\label{app:fixed-state-replay}
The warmed variance sweep in the main text changes both the step size and the state reached by warmup, so the measured input $\sigma_\phi^2(\eta)$ drifts with $\eta$.  The direct fixed-state replay removes that confound.  We warm a Pythia-1B math--code cell once, fork the full optimizer state $(\theta,m,v,s)$, and replay the same $K=16$ ordering orbit from that fork at each learning rate.  Because the local potentials are measured at a single state, raw and $\sigma_\phi^2$-clamped slopes are identical to three decimals.  The measured order-orbit variance slopes are given in Table~\ref{tab:fixed-state-replay}.
\begin{table}[t]
\centering
\small
\caption{Direct fixed-state replay: order-orbit variance slopes from a single warmed Pythia-1B math--code cell with the optimizer state $(\theta,m,v,s)$ forked once and the same $K=16$ ordering orbit replayed at each learning rate.  The buffered arms scale near $\eta^2$ and buffer-free SGD near $\eta^4$.}
\label{tab:fixed-state-replay}
\begin{tabular}{lc}
\toprule
Arm & variance slope \\
\midrule
AdamW & $1.830$ \\
fixed-$\beta$ momentum & $1.999$ \\
SGD & $3.998$ \\
\bottomrule
\end{tabular}
\end{table}
These slopes give the fixed-state version of the class split: a non-adaptive fixed buffer has variance exponent two, buffer-free SGD has variance exponent four, and AdamW is in the first-order variance class with a small preconditioner-path remainder.  The artifact reproduces these values from the fixed-state replay records.

\subsection{Additional optimizer tests}
\label{app:additional-optimizers}
The closed-form floor is the AdamW frozen-$P$ specialization.  The broader structural result is that fixed-clock optimizer state can move equal-multiset order into the first-order class when its first-order replay coefficient is nonzero.  Muon falls under that structural criterion because its orthogonalized momentum is consumed on an $\eta$-independent clock.  Its measured variance slope was $1.56$ with bootstrap interval $[1.31,1.82]$, placing it in the first-order class.  A parallel Lion variance test returned $2.87$ with interval $[2.31,3.94]$, which is inconclusive at this sample size; Lion is used in the main text only as a mean-level fixed-clock exponent comparison.

\subsection{Fixed-target triple controls}
\label{app:triple-controls}
The fixed-target triangle experiment is included to test the scalar-sortability geometry.  Under the common target covector, the conservative frozen-$P$ component has zero directed 3-cycles in all 700 tested cells.  The bracket-only field is cyclic on 7.1\% of triples, and the full edge field is cyclic only where the bracket is large enough to cross the radius.  Below $\rho_{\mathrm{curl}}<1$, the full field has 0/286 directed 3-cycles.

Optimizer-attribution arms are retained as scope controls.  The matched-clock momentum control is potential-dominated on every tested edge ($\rho_{\mathrm{curl}}\le1.5\times10^{-4}$ across all 700 cells) and has zero cycles at every $K$.  Plain SGD has vanishing potential component in this decomposition, leaving a degenerate all-curl edge field with a $K$-invariant cycle set.  The matched-clock and SGD controls support the interpretation that fixed-clock AdamW's finite-clock potential moves edges across the radius; the underlying bracket geometry exists more broadly, but the clock determines whether it is expressed in realized pairwise order preferences.

\subsection{Seed-count worked example}
\label{app:seed-example}
For the Pythia-1B math--code orbit at $\eta=5\times10^{-5}$ and $K=16$, the measured AdamW orbit variance in LoRA-delta readout units is $3.0\times10^{-8}$, so $\sigma_{\mathrm{ord}}\approx1.73\times10^{-4}$.  A comparison gap of $\Delta=2\sigma_{\mathrm{ord}}\approx3.5\times10^{-4}$ therefore needs about four independent shuffle seeds per side at 95\% level and 80\% power by Eq.~\eqref{eq:seed-count}; a gap of one order-noise standard deviation needs about sixteen.  The example is illustrative rather than universal: practitioners should compute $V_K$ for their optimizer clock and estimate $\sigma_\phi$ or run a small ordering pilot for their own setting.

\end{document}